\documentclass{article}
\PassOptionsToPackage{dvipsnames,table}{xcolor}

\usepackage{iclr2027_conference,times}
\iclrfinalcopy

\usepackage{enumitem}

\usepackage{xr-hyper}

\usepackage[utf8]{inputenc} \usepackage[T1]{fontenc}    \usepackage{hyperref}       \usepackage{url}            \usepackage{booktabs}       \usepackage{amsfonts}       \usepackage{nicefrac}       \usepackage{microtype}      \usepackage{graphicx}       \usepackage{float}          \usepackage[dvipsnames,table]{xcolor}
\usepackage[most]{tcolorbox}
\usepackage{wrapfig}
\usepackage{subcaption}
\usepackage{multirow}
\usepackage{makecell}
\usepackage{colortbl}
\definecolor{lightgray}{gray}{0.8}

\definecolor{takeawayframe}{HTML}{2F5C8A}
\definecolor{takeawayback}{HTML}{F2F6FB}
\definecolor{figmagenta}{HTML}{AA3377}
\definecolor{figgrey}{HTML}{999999}
\definecolor{figred}{HTML}{CC3311}
\definecolor{termblue}{HTML}{0B3D91}
\definecolor{termpurple}{HTML}{552288}

\newtcolorbox{takeaway}{enhanced, breakable, sharp corners, boxrule=0pt, leftrule=2.2pt,
  colback=takeawayback, colframe=takeawayframe,
  left=6pt, right=6pt, top=4pt, bottom=4pt,
  before skip=8pt, after skip=8pt}
\usepackage{enumitem}

\usepackage{amsmath,amssymb,amsthm,bm}
\usepackage{mathtools}
\IfFileExists{algorithm2e.sty}{\usepackage{algorithm2e}
  \RestyleAlgo{ruled}
}{\usepackage{algorithm_fallback}
}
\theoremstyle{plain}
\newtheorem{theorem}{Theorem}[section]
\newtheorem{proposition}[theorem]{Proposition}
\newtheorem{lemma}[theorem]{Lemma}
\newtheorem{corollary}[theorem]{Corollary}
\theoremstyle{definition}

\newtheorem{assumption}[theorem]{Assumption}
\theoremstyle{remark}

\newcommand{\1}{\mathbf{1}}
\newcommand{\E}{\mathbb{E}}

\newcommand{\argmax}{\mathop{\mathrm{argmax}}}

\newcommand{\prob}{\mathbb{P}}

\usepackage{xspace}

\newcommand*\circledred[1]{\tikz[baseline=(char.base)]{
  \node[shape=circle, draw=BrickRed!60, fill=BrickRed!10, thick, inner sep=1pt] (char) {\scriptsize\textsf{#1}};
}}
\newcommand*\circledblue[1]{\tikz[baseline=(char.base)]{
  \node[shape=circle, draw=NavyBlue!60, fill=NavyBlue!10, thick, inner sep=1pt] (char) {\scriptsize\textsf{#1}};
}}
\usepackage{pifont}

\DeclareRobustCommand{\redmarker}{\tikz[baseline= -0.6ex]{
    \fill[black] (0,0) circle (0.22ex);
    \draw[figred, line width=1.5pt] (0,0) circle (0.7ex);
  }}

\definecolor{rqblue}{HTML}{5677C8}
\definecolor{rqbluebg}{HTML}{EAF0FF}

\definecolor{rqcyan}{HTML}{4F91A8}
\definecolor{rqcyanbg}{HTML}{E8F5F8}

\definecolor{rqpurple}{HTML}{8A4FA3}
\definecolor{rqpurplebg}{HTML}{F3EAF7}

\newcommand{\rqbadge}[3]{\leavevmode
  \tcbox[
    on line,
    boxsep=0pt,
    left=1.5pt,
    right=1.5pt,
    top=0.5pt,
    bottom=0.15pt,
    arc=2pt,
    boxrule=0.5pt,
    colframe=#1,
    colback=#2
  ]{\raisebox{-0.03ex}{\textcolor{#1}{\strut\scriptsize\bfseries #3}}}}

\newcommand{\rqone}{\rqbadge{rqblue}{rqbluebg}{RQ1}\ }
\newcommand{\rqtwo}{\rqbadge{rqcyan}{rqcyanbg}{RQ2}\ }
\newcommand{\rqthree}{\rqbadge{rqpurple}{rqpurplebg}{RQ3}\ }

\makeatletter
\def\maketag@@@#1{\hbox{\m@th\normalfont\normalsize#1}}
\makeatother

\title{Efficient Best-of-N policy evaluation for inference-time alignment}

\author{
\makebox[0.96\textwidth][c]{Jonas Schweisthal$^{1,*}$ \qquad
Yuxin Wang$^{1,*}$ \qquad
Athiya Deviyani$^{2}$}\\[0.4em]
\makebox[0.96\textwidth][c]{\textbf{Stefan Feuerriegel}$^{1}$ \qquad
\textbf{Dennis Frauen}$^{3}$}\\[0.6em]
correspondence to: \texttt{jonas.schweisthal@lmu.de}, \texttt{yuxin.wang1@lmu.de}
}

\begin{document}

\maketitle

\begingroup
\renewcommand\thefootnote{}
\footnotetext{\small$^{1}$LMU Munich \& MCML, \quad $^{2}$Carnegie Mellon University, \quad $^{3}$Meta (work in personal capacity). \quad $^{*}$Equal contribution.}
\endgroup

\begin{abstract}
Best-of-$N$ (BoN) is a common inference-time alignment method that selects the highest-scoring response among $N$ samples from a reference model. Evaluating BoN policies from logged data is challenging under sample-only access because standard off-policy estimators require density ratios that depend on unavailable response likelihoods. In this paper, we propose a sample-only framework for evaluating and selecting BoN policies without access to these likelihoods. We show that the order-statistic structure of BoN allows the required density ratios to be expressed through score-rank probabilities that are estimable from samples alone. We then develop a doubly robust estimator of the BoN policy value (BoN-DR) that efficiently reuses a shared auxiliary sample pool across candidate budgets. We establish valid asymptotic inference even under reward estimator misspecification and prove the efficiency of our BoN-DR estimator. Since larger budgets can amplify errors in the score function and lead to reward overoptimization, we derive two selection rules: (i)~maximizing the estimated policy value and (ii)~maximizing a lower confidence bound on the improvement over the reference policy, which accounts for estimation uncertainty and provides a no-harm guarantee. Across synthetic experiments and GSM8K with multiple reference and reward models, our framework accurately estimates BoN policy values and selects effective sampling budgets.
\end{abstract}

\section{Introduction}

In \textit{inference-time alignment}, the aim is to improve the behavior of a given generative model using sampled outputs and an auxiliary reward or preference signal, but \textit{without} updating or accessing the model parameters \citep{huang2025best, khalaf2025inferencetime}. This is relevant in practice, where one often interacts with a model through a restricted interface and cannot modify the underlying parameters. For example, many large companies want to adapt frontier LLMs with API-only access to their own preferences and needs \citep{kosmas2025code}, such as a car manufacturer using collected user feedback to improve a third-party in-car assistant.

A widely used approach is \textit{Best-of-$N$} (BoN) sampling \citep{nakano2021webgpt, beirami2025guarantees}, which has two steps: (1)~for a prompt $x$, the reference model (e.g., the LLM) generates $N$ candidate responses; and (2)~a score function $s$ (e.g., an externally pre-trained reward model or an LLM judge) selects the highest-scoring response \citep{nakano2021webgpt, beirami2025guarantees}. Importantly, BoN requires \textbf{only} samples and \textbf{no} modification of the reference model. However, deploying BoN in practice is subject to two challenges:

\textit{\textbf{Challenge} \circledred{1}: Evaluating BoN policies under sample-only access.} Practitioners need to know whether and, in particular, \textit{how much} a BoN policy actually improves the true downstream reward (and not only the score) over the reference model. From logged responses and rewards, this is an \textit{off-policy} evaluation (OPE) problem. Yet, standard OPE estimators (e.g., importance-weighted and doubly robust estimators) are \textbf{not} applicable. The reason is that these estimators require the density ratio between the target policy and the policy that generated the logged data \citep{Dudik.2011}. Under sample-only access, however, the probabilities assigned to complete responses, i.e., the response likelihoods, are unavailable. $\Rightarrow$ \textit{We therefore need an estimator that evaluates BoN policies \textbf{without} evaluating the response likelihoods}.

\textit{\textbf{Challenge} \circledred{2}: Selecting the sampling budget $N$.}
The sampling budget $N$ controls how strongly BoN optimizes the score: a small $N$ limits the potential benefit, while a large $N$ concentrates the policy on low-probability, high-scoring responses. When the score function is an imperfect proxy for the downstream reward, this can amplify its errors and reduce downstream performance. This issue is often known as \textit{reward overoptimization} \citep{gao2023scaling, coste2024reward}. Figure~\ref{fig:motivation} illustrates the issue on GSM8K: the score function of the selected response increases with $N$, while the downstream accuracy can peak and then decline. The score function alone therefore does not reveal which $N$ should be deployed; instead, this requires estimating the policy value of each candidate budget. However, even unbiased value estimates are noisy in finite samples, especially for large $N$, where the BoN policy concentrates on responses that are rarely observed in the logged data. Directly maximizing them can thus lead to harmful budget choices. $\Rightarrow$ \textit{We therefore need a selection rule that accounts for uncertainty when comparing candidate budgets and guards against harmful departures from the reference policy.}

\begin{figure}[t]
    \centering
    \includegraphics[width=\textwidth]{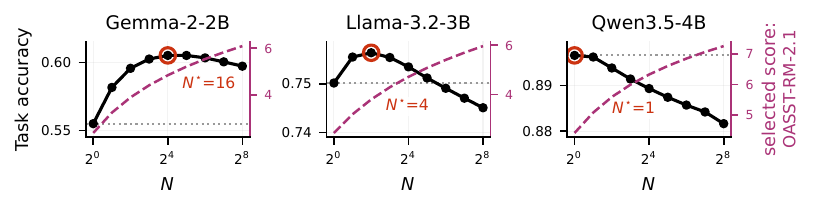}
    \vspace{-1.0cm}
    \caption{\textbf{The score function does not reveal the optimal inference-time budget.} Best-of-$N$ on GSM8K~\citep{cobbe2021training} for three reference models using OASST-RM-2.1 as the score model. \textbf{Black solid:} task accuracy computed from $512$ stored responses for each of the $1{,}319$ held-out test prompts. \textbf{\textcolor{figmagenta}{Magenta dashed}:} mean selected-response score (right axis). \textbf{\textcolor{figgrey}{Grey dotted}:} reference-model accuracy ($N=1$). \redmarker\ Red circles mark the accuracy-optimal budget $N^\star$. $\Rightarrow$~\textbf{Takeaway:} The proxy score increases with $N$, but accuracy peaks and then falls, showing that the score function alone cannot determine the best sampling budget.
    }
    \vspace{-0.7cm}
    \label{fig:motivation}
\end{figure}

In this paper, \textbf{we address both challenges by developing a sample-only framework for evaluating BoN policies and selecting the inference-time sampling budget $N$.} We proceed in two steps:

\circledblue{1}~\textbf{Estimation} ($\rightarrow$ Section~\ref{sec:ope}): To develop an estimator, we exploit a statistical property of BoN: the underlying order-statistic structure makes the otherwise unavailable density ratio between the BoN policy and the reference model expressible through the probabilities that a fresh reference-model sample scores below or at a given response, which can be estimated from samples alone. Building on this, we construct a doubly robust estimator of the BoN policy value (called BoN-DR) \textit{without} access to response likelihoods. Here, BoN-DR separates the score used to \textit{select} responses from the reward estimator used to \textit{evaluate} them. Further, it reuses a shared auxiliary sample pool across candidate budgets and is unbiased for any finite pool size. We show that our BoN-DR yields asymptotically valid inference even when the reward estimator is misspecified, and that it attains the efficiency of an exact-ratio DR estimator as the pool grows and the reward estimator is consistent.

\circledblue{2}~\textbf{Selecting $N$} ($\rightarrow$ Section~\ref{sec:select_N}):
To select the sampling budget $N$, we develop two selection rules: (i)~\textit{Value Maximization} chooses the budget with the largest estimated policy value; and (ii)~\textit{Safe Improvement} chooses the budget with the largest lower confidence bound (LCB) on the improvement over the reference policy, with a fallback to $N=1$ if no LCB is positive. We show that both rules asymptotically attain the best policy value on the candidate grid, while Safe Improvement additionally provides an asymptotic no-harm guarantee relative to the reference policy with high probability. To compute the LCB, we further construct simultaneous confidence bounds across candidate budgets. Unlike pointwise intervals, these bounds account for searching over the candidate grid and remain asymptotically valid after $N$ is selected from the same data.

Our \textbf{contributions} are fourfold:\footnote{Code available at: \url{https://github.com/JSchweisthal/Efficient-BoN}.} \textbf{(1)}~We develop BoN-DR, a novel, unbiased doubly robust estimator for evaluating BoN policies from logged data, which requires only sample access to the reference model. \textbf{(2)}~We prove statistical guarantees for BoN-DR, including robustness to reward estimator misspecification, asymptotic normality, and efficiency. \textbf{(3)}~We propose two principled selection rules for $N$ (i.e., Value Maximization and Safe Improvement) with consistency and a no-harm guarantee. \textbf{(4)}~We evaluate our methods in synthetic experiments and on GSM8K across multiple reference and score models.

\section{Related work\protect\footnote{We provide an extended related work in Appendix~\ref{app:related_work}}.}
\label{sec:related_work}

\textbf{BoN for inference-time alignment.}
BoN is widely used for inference-time alignment with reward or preference models \citep{stiennon2020learning, nakano2021webgpt}. Prior work studies the induced policy, its divergence to the reference policy, and the theoretical and empirical scaling behavior \citep[e.g.,][]{gui2024bonbon, yang2024asymptotics, mroueh2025information, beirami2025guarantees, sriraman2026revisiting}. Larger budgets can amplify proxy errors and cause reward overoptimization, motivating methods that modify the score, selection rule, or underlying model \citep{gao2023scaling,coste2024reward,huang2025best,jinnai2025regularized}. We instead keep standard BoN fixed and study how to evaluate its downstream value and select its budget from logged data. \citet{khalaf2025inferencetime} consider a closely related problem, where they tune the BoN budget using proxy and true rewards of sampled responses. However, their method requires true rewards for all sampled candidates and does \textbf{not} provide uncertainty quantification. In contrast, we select $N$ from logged data with only one observed reward per prompt and provide statistical guarantees.

\textbf{Off-policy evaluation (OPE).}
Classical OPE estimates a target policy's value from data collected under a behavior policy using replay, importance weighting, or doubly robust estimation. \citep{li2011unbiased,horvitz1952generalization,strehl2010learning,Dudik.2011}. High-confidence OPE and safe policy improvement further use lower confidence bounds to guard against underperforming a baseline \citep{thomas2015highOPE, thomas2015highimprove, laroche2019safe}. These methods cannot be directly applied to BoN under sample-only access because the required target-to-behavior density ratio depends on \textit{unavailable} response likelihoods. We later exploit the order-statistic structure in BoN to express this ratio through score-rank probabilities, thereby enabling policy value estimation and uncertainty-aware budget selection from sampled responses and logged rewards.

\section{Problem Setup}

\vspace{-0.1cm}
We consider a context $X \in \mathcal{X}$, an action $A \in \mathcal{A}$, and a scalar reward $Y \in \mathcal{Y}$. We call any (possibly stochastic) decision rule $\pi$ mapping a context to an action a \textit{policy}. We assume that the observed action $A$ is generated by a \textit{behavior} or \textit{reference policy} $\pi_b(a \mid x) = \prob(A=a | X=x)$. For example, in the LLM setting, $\pi_b$ is induced by a \textit{reference model} that generates a text $A$ given a prompt $X$, based on which we obtain downstream signal $Y$ (e.g., task success, human preference, click-through rate, etc.). We denote the reward function as $\mu(x, a) = \E[Y | X=x, A=a]$, and we observe an i.i.d.\ logged dataset $\mathcal D=\{(x_i,a_i,y_i)\}_{i=1}^n$ from the population $Z = (X, A, Y) \sim \mathbb{P}$.

\textbf{Inference-time alignment.} We assume that we can \textit{sample} from $\pi_b(\cdot\mid x)$ for any context $x$, but can \underline{neither} access the model weights \underline{nor} evaluate the response likelihood $\pi_b(a\mid x)$. This is typical for API-based models, which return sampled responses but no token-level probabilities. Our goal is to improve the deployed policy using only such samples and the logged data $\mathcal D$.

\textbf{BoN policies.} A well-established method for inference-time alignment is \textit{best-of-N sampling} (BoN) \citep{nakano2021webgpt, beirami2025guarantees}. Here, a score function $s(x,a)$ first ranks actions $a$ for each context $x$. For example, $s(x,a)$ could be an LLM judge \citep{zheng2023judging} or an estimator of the reward function $\mu(x, a)$. Then, for a fixed context $x_i \in \mathcal{X}$, BoN samples $N$ i.i.d. actions
$a_{i1},\dots,a_{iN} \sim \pi_b(\cdot \mid x_i)$ from the API and chooses the one with the largest score
$a^{\star,s}_N(x_i) \in \arg\max_{j \leq N} s(x_i,a_{i,j})$.
The induced BoN policy is defined as $\pi_N^s(a\mid x_i)
:=
\prob\!\left(a^{\star,s}_N(x_i)=a \mid X=x_i \right).$
In particular, $\pi_1^s=\pi_b$, whereas increasing $N$ progressively concentrates the policy on actions receiving high scores under $s$.

\textbf{Policy evaluation.} Our paper primarily focuses on the \textit{evaluation} of the BoN policy $\pi_N^s$. For example, we might be interested in the performance of $\pi_N^s$ for a specific $N$, or whether $\pi_N^s$ for $N>1$ improves over the base model $\pi_b$. A principled way to evaluate a policy $\pi$ is via its \textit{policy value} $V(\pi)
= \E\!\Big[ \sum_{a \in \mathcal{A}} \pi(a\mid X)\mu(X,a) \Big]
$,
which quantifies the expected reward under policy $\pi$ \citep{Kallus.2018}.

\textbf{The choice of $N$.} Evaluating $V(\pi_N^s)$ is directly tied to the choice of $N$, which controls how far $\pi_N^s$ departs from $\pi_b$: a small $N$ yields only marginal gains, while a large $N$ concentrates the policy on high-scoring actions that are rare under $\pi_b$ and thus weakly represented in $\mathcal D$. As a result, errors in the score function can be amplified, which can lead to reward overoptimization \citep{gao2023scaling}.

In this paper, we address two core challenges:
\begin{tcolorbox}[
  colback=takeawayback,
  colframe=takeawayframe,
  boxrule=0.6pt,
  arc=1pt,
  left=5pt,right=5pt,top=4pt,bottom=4pt,
  boxsep=0pt,
  before skip=5pt,
  after skip=5pt
]
\circledblue{1}~\textbf{Off-policy evaluation} ($\rightarrow$ Section~\ref{sec:ope}): How can we efficiently estimate the value $V(\pi_N^s)$ of a BoN policy from data $\mathcal D$ generated by $\pi_b$ when the reference model can only be sampled?

\circledblue{2}~\textbf{Selecting $N$} ($\rightarrow$ Section~\ref{sec:select_N}): How can we use these estimates to select the sampling budget $N$ with statistical guarantees?
\end{tcolorbox}

\section{Efficient Off-Policy Evaluation for BoN}
\label{sec:ope}

In this section, we develop an estimator for the value of the BoN policy $V(\pi_N^s)$. A natural starting point is the doubly robust (DR) estimator, which is widely used for off-policy evaluation due to its favorable statistical properties \citep{Robins.1994b, Dudik.2011}. For a fixed BoN policy $\pi_N^s$, the DR estimator is
\vspace{-0.15cm}
\begin{equation}\label{eq:dr-bon}
\hat V_{\mathrm{DR}}(\pi_N^s) = \frac{1}{n}\sum_{i=1}^n \psi_N(z_i), \psi_N(z_i) = \sum_{a\in\mathcal A} \textcolor{termblue}{\pi_N^s(a\mid x_i)\hat\mu(x_i,a)} + \textcolor{termpurple}{\frac{\pi_N^s(a_i\mid x_i)} {\pi_b(a_i\mid x_i)} \bigl(y_i-\hat\mu(x_i,a_i)\bigr)},
\vspace{-0.2cm}
\end{equation}
where $z_i=(x_i,a_i,y_i)$ and $\hat\mu$ is an estimator of the reward function $\mu$.
The \textcolor{termblue}{plug-in} term (first term in the sum) evaluates the target policy under the estimated reward function, while the \textcolor{termpurple}{correction term} (second term) uses the observed outcome to correct for errors in the estimated reward function, weighted by the density ratio between the BoN and behavior policies.

\textit{Why the standard DR estimator is not applicable.} Direct application of the DR estimator in the form of \eqref{eq:dr-bon} is \textit{not} possible in the inference-time alignment setting due to two main challenges: (i)~both terms require knowledge of $\pi_N^s(a\mid x_i)$, but the induced BoN probabilities are not available from a black-box sampler. Attempting to estimate $\pi_N^s(a\mid x_i)$ from the data is extremely challenging in the enormous and sparse text action space $\mathcal{A}$. (ii)~The correction term requires the density ratio $w_N(x,a) := \frac{\pi_N^s(a\mid x)}{\pi_b(a\mid x)}$.
Because the action space $\mathcal A$ consists of high-dimensional text sequences, individual sequence probabilities can be extremely small, and density-ratio estimation can suffer from high variance and numerical instability.

To address this, we first establish an alternative representation of the DR score for the inference-time alignment setting (Section~\ref{sec:dr_score}). We then construct a pooled estimator that reuses auxiliary samples across candidate budgets (Section~\ref{sec:bon_dr_estimator}) and establish theoretical guarantees (Section~\ref{sec:theory}).

\subsection{Inference-time DR score}
\label{sec:dr_score}

Our key observation is that the target policy is not arbitrary: $\pi_N^s$ is obtained by applying a known Best-of-$N$ selection rule to samples from $\pi_b$. We can therefore exploit two pieces of structure jointly: (i)~although $\pi_b$ cannot be evaluated, we can sample from $\pi_b(\cdot\mid x)$ at inference time, and (ii)~the construction of the BoN policy imposes a particular relationship between $\pi_N^s$ and $\pi_b$.

In the following, we show that DR score admits a specific representation only in terms of one-dimensional score ranks under the behavior policy.
For fixed $x$ and action $a$, we define the strict and weak cumulative density functions (CDFs) of the score via \vspace{-0.1cm}
\begin{equation}
F_x^-(a)
=
\prob_{A\sim \pi_b(\cdot\mid x)}
\left(
s(x,A)<s(x,a)
\right) \text{ and } F_x(a)
=
\prob_{A\sim \pi_b(\cdot\mid x)}
\left(
s(x,A)\leq s(x,a)
\right).
\vspace{-0.3cm}
\end{equation}

Thus, $F_x^-(a)$ and $F_x(a)$ describe the rank of $a$ according to the score $s$ among actions sampled from the behavior policy. Importantly, both quantities are probabilities of simple comparison events and can therefore be estimated using samples from $\pi_b(\cdot\mid x)$ without evaluating any sequence likelihoods. The following proposition shows that these rank probabilities are sufficient to express the entire DR score for a BoN policy (which we call BoN-DR score).

\begin{proposition}[BoN-DR score]
\label{prop:bon-dr-score}
Suppose that distinct actions have distinct scores, i.e., \vspace{-0.1cm}
$s(x,a)=s(x,a')
\quad\Longrightarrow\quad
a=a',$
for any $a,a' \in \mathcal A$. Define $g_N(u,v) =
\sum_{\ell=0}^{N-1} u^{N-1-\ell}v^\ell$.
Then, for $z=(x,a,y)$ with $\pi_b(a\mid x)>0$, the DR score for the BoN policy can be written as \begin{equation}\label{eq:bon-dr-score}
\psi_N(z)
=
\E_{A' \sim \pi_b(\cdot\mid x)}
\left[
g_N\!\left(F_x(A'),F_x^-(A')\right)
\hat\mu(x,A')
\right] +
g_N\!\left(F_x(a),F_x^-(a)\right)
\bigl(y-\hat\mu(x,a)\bigr).
\end{equation}
\end{proposition}

\begin{proof}
    This follows by combining the BoN order-statistic identity in \citet{beirami2025guarantees} with the DR score formula in Eq.~\eqref{eq:dr-bon}. We refer to Appendix~\ref{app:proofs} for details.
\vspace{-0.2cm}
\end{proof}

Proposition~\ref{prop:bon-dr-score} resolves the two computational obstacles that we identified above: (i)~it eliminates the need to evaluate either $\pi_b(a\mid x)$ or $\pi_N^s(a\mid x)$; and (ii) the density ratio is replaced by $g_N(F_x(a),F_x^-(a))$, which depends only on the relative rank of $a$ under the score distribution induced by the behavior policy. Consequently, apart from evaluating the score function $s$ and outcome model $\hat\mu$, every quantity in Eq.~\eqref{eq:bon-dr-score} can be approximated using only samples from the behavior policy. In particular, $F_x(a)$ and $F_x^-(a)$ are expectations of bounded indicator variables, e.g., $F_x(a) = \E_{A' \sim \pi_b(\cdot\mid x)} \left[ \1\{s(x,A')\leq s(x,a)\} \right]$, thus turning the original high-dimensional density-ratio evaluation problem into the estimation of simple rank probabilities.

The remaining challenge is therefore how we should estimate the expectations and score-rank functionals in Eq.~\eqref{eq:bon-dr-score} given a finite number of auxiliary samples from $\pi_b(\cdot\mid x)$. In the next section, we show that their particular polynomial structure admits natural unbiased U-statistic estimators.

\subsection{Pooled BoN-DR estimator}
\label{sec:bon_dr_estimator}

We now construct an estimator of the DR score in Eq.~\eqref{eq:dr-bon} using a finite, reusable pool of samples from the behavior policy. Proposition~\ref{prop:bon-dr-score} shows that the DR score can be estimated using only samples from the behavior policy. A direct Monte Carlo implementation of Eq.~\eqref{eq:bon-dr-score} could draw fresh samples from the behavior policy each time an expectation or score rank needs to be estimated, and repeat this procedure separately for every candidate value of $N$. For generative models, however, sampling may be computationally costly.

Instead, we propose to draw a single reusable pool of samples for each context and use it to estimate both components of the DR score (i.e., the \textcolor{termblue}{plug-in term} and the \textcolor{termpurple}{correction term}), as well as the values of other BoN policies with different $N$. For each context $x_i$, we draw one pool of behavior-policy samples $\tilde a_i = (\tilde a_{i1},\dots,\tilde a_{iM})
\overset{\textup{iid}}{\sim}
\pi_b(\cdot\mid x_i)$,
where $M$ is a user-chosen computational budget, and let us denote $\bar a_i = (a_i, \tilde a_{i1},\dots,\tilde a_{iM}) \in \mathcal{A}^{M+1}$ the vector of Monte Carlo samples and behavior actions.

To derive our BoN-DR estimator, we leverage \textit{U-statistics} \citep{vanderVaart.1998}, which estimate a quantity defined as an expectation over $r$ i.i.d.\ draws via a single larger pool of $M$ draws. The basic idea is to estimate the quantity by averaging the same quantity over all size-$r$ subsets of the pool. Each subset has the correct distribution, so the resulting estimator is unbiased while reusing all available samples. Below, we apply this construction separately to (i)~the \textcolor{termblue}{plug-in} term and (ii)~\textcolor{termpurple}{correction term} of the DR score and then combine the two into our pooled BoN-DR estimator.

\textbf{Estimator for the \textcolor{termblue}{plug-in term}.}
By Proposition~\ref{prop:bon-dr-score}, the plug-in term can be written as
\begin{equation}
v_N^s(x_i)
=
\E_{A\sim\pi_N^s(\cdot\mid x_i)}
\left[
\hat\mu(x_i,A)
\right]
=
\E\left[
\hat\mu(x_i,A_N^{\star,s})
\mid X_i=x_i
\right],
\label{eq:plugin-bon-representation}
\end{equation}
where $A_N^{\star,s}$ is the action selected from $N$ i.i.d.\ draws from
$\pi_b(\cdot\mid x_i)$ according to the score $s$. Since $\bar a_i$ contains $M+1$ such draws, we estimate $v_{i,N}^s$ by averaging the fitted reward of the BoN winner over all size-$N$ subsets: \vspace{-0.1cm}
\begin{equation}
\hat v_{i,N}^{s,M+1} = \binom{M+1}{N}^{-1} \sum_{\substack{ J\subseteq\{0,\ldots,M\},  |J|=N }} \hat\mu\left( x_i,\bar a_{i,J}^{\star,s} \right). \vspace{-0.1cm}
\label{eq:plugin-u-stat} \end{equation}
This estimator can be evaluated without enumerating all subsets. Let $\mathcal G_i$ index the unique actions appearing in $\bar a_i$, and let $\bar a_{ig}$ denote the corresponding action. We define \vspace{-0.3cm}
\begin{equation}
L_{ig} = \sum_{m=0}^{M} \1\left\{ s(x_i,\bar a_{im}) < s(x_i,\bar a_{ig}) \right\} \text{ and } E_{ig} = \sum_{m=0}^{M} \1\left\{ s(x_i,\bar a_{im}) \leq s(x_i,\bar a_{ig}) \right\}.
\end{equation}
Then, an alternative estimator is \vspace{-0.3cm}
{\begin{equation}
\textcolor{termblue}{\hat v_{i,N}^{s,M+1}(x_i, \bar a_i)} = \sum_{g\in\mathcal G_i} \hat\mu(x_i,\bar a_{ig}) \frac{ \binom{E_{ig}}{N} - \binom{L_{ig}}{N} }{ \binom{M+1}{N} }. \label{eq:plugin-pooled}
\vspace{-0.1cm}
\end{equation} }
Indeed, the numerator counts exactly the size-$N$ subsets for which $\bar a_{ig}$ is the highest-scoring action. We use the convention $\binom{u}{N}=0$ for $u<N$.

\textbf{Estimator for the \textcolor{termpurple}{correction term}.} For the correction term, we treat the logged action $a_i$ as fixed and use the $M$ auxiliary samples to estimate $w_N^s(x_i,a_i) = g_N(F_{x_i}(a_i),F_{x_i}^-(a_i))$. We define
\vspace{-0.1cm}
\begin{equation} L_i = \sum\nolimits_{m=1}^{M} \1\left\{ s(x_i,\tilde a_{im}) < s(x_i,a_i) \right\}, \text{ and } E_i = \sum\nolimits_{m=1}^{M} \1\left\{ s(x_i,\tilde a_{im}) \leq s(x_i,a_i) \right\}.
\vspace{-0.1cm}
\end{equation}
Under a no-score-ties assumption, $E_i-L_i$ counts the number of auxiliary draws equal to $a_i$. By averaging over all size-$(N-1)$ subsets of the auxiliary pool, we yield the estimator \vspace{-0.1cm}
\begin{equation}
\textcolor{termpurple}{\hat w_{i,N}^{s,M}(x_i, \bar a_i)} = N \sum_{t=0}^{\min\{E_i-L_i,N-1\}} \frac{1}{t+1} \frac{ \binom{E_i-L_i}{t} \binom{L_i}{N-1-t} }{ \binom{M}{N-1} }.
\vspace{-0.1cm}
\label{eq:weight-pooled}
\end{equation}
Here, $t$ denotes the number of additional copies of $a_i$ among the $N-1$ sampled competitors, while the factor $1/(t+1)$ accounts for the $t+1$ identical occurrences of the winning action.

\begin{tcolorbox}[
  colback=takeawayback,
  colframe=takeawayframe,
  boxrule=0.6pt,
  arc=1pt,
  left=5pt,right=5pt,top=4pt,bottom=4pt,
  boxsep=0pt,
  before skip=5pt,
  after skip=5pt
]
\textbf{Pooled BoN-DR estimator.} We define our \textbf{pooled BoN-DR estimator} as
\begin{align}
\psi_N^M(z_i,\tilde a_i) &= \textcolor{termblue}{\hat v_{i,N}^{s,M+1}(x_i, \bar a_i)} + \textcolor{termpurple}{\hat w_{i,N}^{s,M}(x_i, \bar a_i) \left( y_i-\hat\mu(x_i,a_i) \right)}, \\
\hat V_{\mathrm{DR}}^M(\pi_N^s) &= \frac{1}{n} \sum_{i=1}^n \psi_N^M(z_i,\tilde a_i).
\vspace{-0.2cm}
\label{eq:pooled-dr-estimator}
\end{align}
Thus, a single pool of $M$ additional generations per context is sufficient to estimate both components of the DR score. Moreover, the same pool can be reused across different values of $N$, which is particularly important for the policy-selection problem considered in the following section.
\end{tcolorbox}

\subsection{Theoretical guarantees}
\label{sec:theory}

We now establish asymptotic guarantees for our pooled BoN-DR estimator, including valid inference and efficiency. For this, we assume that $s$ and $\hat\mu$ are constructed independently of $\mathcal D$ (e.g., external models or a sample split), together with standard sampling and moment conditions (Assumption~\ref{ass:regularity}).

\begin{theorem}[Validity and efficiency of the pooled BoN-DR estimator]
\label{thm:dr-asymptotics}
Fix $N$ and condition on the training sample used to construct $s$ and
$\hat\mu$.

\textit{\textbf{(a)} Finite-$M$ validity.}
For any fixed $M\geq N-1$,
$
\sqrt n
\left(
\hat V_{\mathrm{DR}}^M(\pi_N^s)
-
V(\pi_N^s)
\right)
\rightsquigarrow
\mathcal N(0,\sigma_{N,M}^2)
$,
where
$\sigma_{N,M}^2
=
\mathrm{Var}\left(
\psi_N^M(Z,\tilde A)
\right)$.
This result does not require $\hat\mu$ to be a consistent estimator of
$\mu$.

\textit{\textbf{(b)} Infinite-$M$ efficiency.}
Suppose additionally that $M=M_n\to\infty$ and that the reward function estimator $\hat{\mu}$
is consistent for the true reward function $\mu$ in $L_2$. Then
$\sqrt n
\left(
\hat V_{\mathrm{DR}}^{M_n}(\pi_N^s)
-
V(\pi_N^s)
\right)
\rightsquigarrow
\mathcal N(0,\sigma_{N,\mathrm{eff}}^2)$,
where $\sigma_{N,\mathrm{eff}}^2$ is the variance of the efficient influence
function
$\phi_N^{\mathrm{eff}}(Z)
=
\sum_{a\in\mathcal A}
\pi_N^s(a\mid X)\mu(X,a)
+
w_N^s(X,A)
\bigl(Y-\mu(X,A)\bigr)
-
V(\pi_N^s)$.

Thus, conditional on the learned policy and behavior sampler, the proposed estimator attains the same first-order efficiency bound as an oracle DR estimator with access to the exact BoN density ratio.  See Appendix~\ref{app:proofs} for the proof.
\end{theorem}

Part~\textbf{(a)} is particularly important for inference-time alignment. Valid inference does \textit{not} require the reward estimator to approximate the true conditional mean. For example, $s$ and $\hat\mu$ may be given by an LLM judge that is useful for ranking responses but is \textit{not} a statistically consistent model of $\E[Y\mid X,A]$. The DR correction still yields a consistent, root-$n$ estimator of the value of the resulting BoN policy. Part~\textbf{(b)} adds that, if $\hat\mu$ is additionally consistent, the BoN-DR estimator attains the corresponding first-order efficiency bound.

We next clarify how the auxiliary sampling budget $M$ affects the estimator. Because the pooled estimators are unbiased for every finite $M$, increasing $M$ removes Monte Carlo variance rather than a first-order bias. For fixed $N$, under the boundedness conditions above, it is sufficient that $M_n\to\infty$; no particular rate of growth relative to $n$ is required.

\section{Selecting the optimal sampling budget $N$ for BoN}
\label{sec:select_N}

Our goal is now to select the optimal sampling budget $N_j$ from a candidate grid $\mathcal{N}=\{N_1,\dots,N_J\}$. In the following, we introduce two selection rules: \textbf{(i)}~\textbf{Value Maximization}, which selects the budget with the largest estimated policy value, and \textbf{(ii)}~\textbf{Safe Improvement}, which selects the budget with the largest lower confidence bound (LCB) on improvement over the reference policy and falls back to $N=1$ when no positive improvement is supported by the data. The latter thus guarantees improvement over the base model with high probability.

\subsection{Value Maximization}

\begin{tcolorbox}[
  colback=takeawayback,
  colframe=takeawayframe,
  boxrule=0.6pt,
  arc=1pt,
  left=5pt,right=5pt,top=4pt,bottom=4pt,
  boxsep=0pt,
  before skip=5pt,
  after skip=5pt
]
\textit{Value Maximization} selects the budget with the largest estimated policy value: \vspace{-0.1cm}
\begin{equation}
\hat N_{\mathrm{value}}
\in
\arg\max_{N\in\mathcal N}
\hat V_{\mathrm{DR}}^M(\pi_N^s).
\vspace{-0.1cm}
\end{equation}
\end{tcolorbox}

The following corollary shows that, over a fixed candidate grid, this rule asymptotically recovers the best policy value in the class.

\begin{corollary}[Oracle consistency of Value Maximization]
\label{thm:value-maximization}
Suppose that the conditions of Theorem~\ref{thm:dr-asymptotics}(a)
hold uniformly over the finite candidate grid $\mathcal N$.
Then, conditional on the training sample,
$V(\pi_{\hat N_{\mathrm{value}}}^s)
\overset{p}{\longrightarrow}
\max_{N\in\mathcal N}V(\pi_N^s)$. The proof is in Appendix~\ref{app:proofs}.
\end{corollary}
\vspace{-0.2cm}
Importantly, this conclusion inherits the robustness of the BoN-DR estimator: it does \textit{not} require $\hat\mu$ to consistently estimate the true reward function. However, the result is an \textit{oracle optimality} guarantee, and, in finite samples, estimation noise can still cause $\hat N_{\mathrm{value}}$ to select a policy for which the value is below $V(\pi_b)$. This motivates our selection rule for safe improvement next.

\subsection{Safe Improvement}

To obtain a no-harm guarantee relative to the behavior policy, we now formulate the selection rule directly in terms of improvement over the behavior policy. To do so, we define \vspace{-0.1cm}
\begin{equation}
\Delta_N
:=
V(\pi_N^s)-V(\pi_1^s)
=
V(\pi_N^s)-V(\pi_b)
\vspace{-0.1cm}
\end{equation}
for each $N\in\mathcal N\setminus\{1\}$. \textit{Safe Improvement} then selects a candidate $N>1$ only when its simultaneous lower confidence bound for $\Delta_N$ is positive; otherwise, we fall back to $N=1$. We use simultaneous rather than pointwise bounds because $N$ is selected adaptively from the same candidate grid.

$\bullet$\,To construct the LCB, we estimate each candidate policy's improvement over the reference policy and its sampling variability. To do so, for each $N\in\mathcal N\setminus\{1\}$, we define the doubly robust improvement score
$D_{i,N}^M
:=
\psi_N^M(z_i,\tilde a_i)
-
\psi_1^M(z_i,\tilde a_i)$,
and the corresponding estimator
$\hat\Delta_N^M
:=
\frac{1}{n}
\sum_{i=1}^n
D_{i,N}^M
=
\hat V_{\mathrm{DR}}^M(\pi_N^s)
-
\hat V_{\mathrm{DR}}^M(\pi_1^s)$.
The corresponding variance is estimated by
$\hat\tau_{N,M}^2
:=
\frac{1}{n}
\sum_{i=1}^n
\big(
D_{i,N}^M-\hat\Delta_N^M
\big)^2$.

$\bullet$\,To obtain \textit{simultaneous} LCBs across candidate budgets, we need to account for the dependence induced by reusing the same evaluation data and auxiliary samples pools across all $N\in\mathcal N$. We do so by using a joint multiplier bootstrap~\citep{chernozhukov2013gaussian}. Conditional on the sample, we thus draw independent multipliers $\xi_i$ satisfying $\E[\xi_i]=0$ and $\E[\xi_i^2]=1$, such as standard Gaussian variables. Formally, we define \vspace{-0.4cm}
{\small\begin{equation}
T_\Delta^{*,+}
:=
\max_{N\in\mathcal N\setminus\{1\}}
\frac{
 n^{-1/2}
 \sum_{i=1}^n
 \xi_i\left(D_{i,N}^M-\hat\Delta_N^M\right)
}{
\hat\tau_{N,M}
}.
\vspace{-0.2cm}
\end{equation}}
Let $\hat c_{1-\alpha}^{\Delta,+}$ denote the conditional $(1-\alpha)$-quantile of $T_\Delta^{*,+}$. Thus, simultaneous LCB is given by \vspace{-0.2cm}
\begin{equation}
\widehat{\mathrm{LCB}}_\Delta(N)
:=
\hat\Delta_N^M
-
\hat c_{1-\alpha}^{\Delta,+}
\frac{\hat\tau_{N,M}}{\sqrt{n}},
\qquad
N\in\mathcal N\setminus\{1\}.
\vspace{-0.3cm}
\end{equation}

\begin{tcolorbox}[
  colback=takeawayback,
  colframe=takeawayframe,
  boxrule=0.6pt,
  arc=1pt,
  left=5pt,right=5pt,top=4pt,bottom=4pt,
  boxsep=0pt,
  before skip=5pt,
  after skip=5pt
]
\textit{Safe Improvement} selects the candidate with the strongest LCB on improvement over the reference policy and falls back to $N=1$ when no candidate has a positive LCB: \vspace{-0.2cm}
\begin{equation}
\hat N_{\mathrm{LCB}}
:=
\begin{cases}
\displaystyle
\arg\max_{N\in\mathcal N\setminus\{1\}}
\widehat{\mathrm{LCB}}_\Delta(N),
&
\displaystyle
\max_{N\in\mathcal N\setminus\{1\}}
\widehat{\mathrm{LCB}}_\Delta(N)>0,
\\[0.7em]
1,
&
\text{otherwise}.
\end{cases}
\end{equation}
\end{tcolorbox}

\begin{theorem}[Safe Improvement (``no-harm guarantee'')]
\label{thm:safe-improvement}
Suppose that the conditions of Theorem~\ref{thm:dr-asymptotics}(a) hold uniformly over the finite candidate grid $\mathcal N$ and that the variances $\tau_{N,M}^2$ are uniformly bounded away from zero.
Conditional on the training sample, it follows that
$V(\pi_{\hat N_{\mathrm{LCB}}}^s)
\overset{p}{\longrightarrow}
\max_{N\in\mathcal N}V(\pi_N^s) $ and
$\prob\left(
V(\pi_{\hat N_{\mathrm{LCB}}}^s)
\geq
V(\pi_b)
\right)
\geq
1-\alpha-o(1)$. Proof in Appendix~\ref{app:proofs}.
\end{theorem}

Theorem~\ref{thm:safe-improvement} formalizes the no-harm property of the Safe Improvement selection rule. A candidate $N>1$ is selected \textit{only} when the improvement over the behavior policy is supported by the simultaneous LCB; otherwise, the rule defaults to $N=1$. The guarantee combines asymptotic no-harm with consistency; i.e., as the sample size grows, the conservatism induced by the LCB vanishes, and Safe Improvement recovers the best policy value over the fixed candidate grid. In finite samples, Safe Improvement may select smaller budgets than Value Maximization because it explicitly accounts for estimation uncertainty.

\section{Experiments}
\label{sec:experiments}

The purpose of our experiments is primarily to verify our theoretical properties. \rqone We use synthetic experiments to test whether our BoN-DR estimator remains \textbf{robust to reward estimator misspecification} while benefiting from informative reward predictions. \rqtwo We assess whether the estimators accurately \textbf{recover BoN policy values} across different reference models and score function regimes in realistic LLM settings. \rqthree We  test whether Safe Improvement reduces harmful choices of $N$ under reward overoptimization and thus leads to \textbf{safe inference-time budget selection}.

\textbf{Settings.} We consider two settings: \textbf{(1)}~\textbf{Synthetic experiments} with known ground truth, following standard practice in the OPE literature~\citep{Dudik.2011, saito2022off}. Here, the score function errors and reward estimator errors can be varied independently through parameters $\gamma$ (which increases the scores of rarely sampled, low-quality actions and thereby changes the induced BoN policy) and $\delta$ (which introduces errors in the predicted rewards without changing the induced policy). Appendix~\ref{app:controlled-details} provides details. \textbf{(2)}~\textbf{Realistic LLM evaluation setting} using GSM8K~\citep{cobbe2021training}, with three reference models (i.e., Gemma-2-2B-IT~\citep{gemma2024gemma2}, Llama-3.2-3B-Instruct~\citep{dubey2024llama3}, and Qwen3.5-4B) and two external reward models as score functions (i.e., OASST-RM-2.1-Pythia-1.4B~\citep{kopf2023openassistant} and Skywork-Reward-V2-Qwen3-1.7B). Details are in Appendix~\ref{app:gsm8k-details}. The true reward is the exact-match correctness (see Appendix~\ref{app:metrics}). Each prompt has a stored pool of $512$ responses. In each evaluation, one response is treated as the logged observation where the reward is revealed, while the remaining responses form the auxiliary pool (see Appendix~\ref{app:efficiency}).

\textbf{Baselines.\footnote{Additional details on the BoN-IPW and plug-in estimators are provided in Appendix~\ref{app:baseline_estimators}.}} To our knowledge, there are \underline{\textbf{no}} existing baselines that directly address OPE estimation for BoN policies from logged data. Instead, we compare three estimators based on Eq.~\eqref{eq:pooled-dr-estimator}: \textbf{(i) Plug-in} averages predictions of reward estimator under the target policy; \textbf{(ii) BoN-IPW} weights logged rewards using estimated rank probabilities; and \textbf{(iii) BoN-DR} uses an additional doubly robust correction. All estimators use the \textit{same} logged data and auxiliary samples. For selection, we use the rules from Section~\ref{sec:select_N}; i.e., \textbf{Value Maximization} and \textbf{Safe Improvement}.

\textbf{Performance metrics.} We measure bias, RMSE, interval width, and coverage of nominal $95\%$ pointwise confidence intervals. Selection is assessed by mean gain over $N=1$ and harm rate, the fraction of repetitions with negative gain.

\begin{figure*}[htbp]
    \centering
    \vspace{-0.2cm}
    \begin{minipage}[t]{0.56\textwidth}
        \vspace{-0.2cm}
        \centering
        \includegraphics[width=\linewidth]{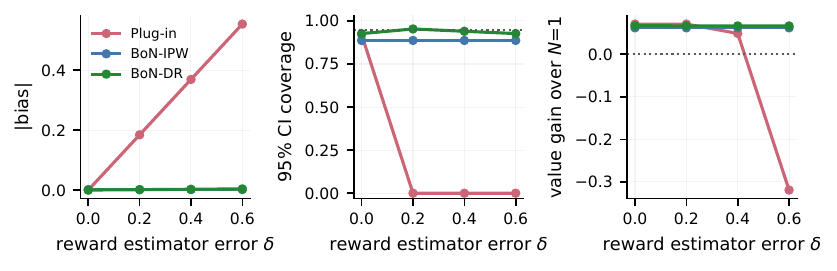}
        \vspace{-0.7cm}
        \captionof{figure}{
        \textbf{Robust to reward estimator misspecification.} Estimated BoN policy value at $N = 256$. We fix $\gamma = 0.65$ and vary $\delta$ from 0 to 0.6 to increase the reward estimator error. Shown:
        \textbf{(a)} Absolute bias of policy value estimation.
        \textbf{(b)} Empirical coverage of the nominal $95\%$ pointwise CI, the dotted line shows the nominal level.
        \textbf{(c)} Gain of the budget selected by Safe Improvement, where the dotted line indicates no improvement over $N=1$.
        }
        \label{fig:double-robustness}
    \end{minipage}
    \hfill
    \begin{minipage}[t]{0.40\textwidth}
        \vspace{0pt}
        \centering
        \small
        \resizebox{\linewidth}{!}{\begin{tabular}{llrrr}
\toprule
$N$ & Estimator & |Bias| & CI width & Coverage \\
\midrule
\multirow{3}{*}{8} & Plug-in & 0.049 & 0.003 & 0.000 \\
 & BoN-IPW & 0.002 & 0.097 & 0.940 \\
 & BoN-DR & \textbf{0.001} & 0.052 & \textbf{0.947} \\
\midrule
\multirow{3}{*}{256} & Plug-in & 0.369 & 0.004 & 0.000 \\
 & BoN-IPW & \textbf{0.002} & 0.377 & 0.887 \\
 & BoN-DR & 0.003 & 0.455 & \textbf{0.940} \\
\bottomrule
\end{tabular}
        }
        \parbox{\linewidth}{\raggedright{\tiny{\textsuperscript{*} CI width and empirical coverage are computed for nominal $95\%$ pointwise CIs.}}}
        \captionof{table}{
        \textbf{Sensitivity to sampling budget $N$.} Here: $\gamma=0.65$, $\delta = 0.4$.
        \textbf{Bold} shows the lowest absolute bias and the coverage closest to $95\%$ for each $N$.
        }
        \label{tab:controlled-calibration}
    \end{minipage}
\vspace{-0.4cm}
\end{figure*}

\rqone \textbf{Robustness to reward estimator misspecification.}
We evaluate robustness to reward estimator misspecification by varying $\delta$, which introduces errors in the reward estimator without changing the induced policy. We compare Plug-in, BoN-IPW, and BoN-DR in terms of absolute bias, CI coverage, and policy value gain over $N = 1$. Results are in Figure~\ref{fig:double-robustness}. Plug-in becomes biased as the errors in the reward prediction are increased. In contrast, BoN-DR and BoN-IPW remain nearly unbiased as $\delta$ increases. Further, they maintain CI coverage close to the nominal 95\% level across the considered levels of misspecification. \textit{$\Rightarrow$ As suggested by our theory, BoN-DR remains robust to reward model misspecification while supporting valid CIs.}

We next examine whether these results persist across different BoN budgets by comparing $N=8$ and $N=256$ in Table~\ref{tab:controlled-calibration}. \textit{$\Rightarrow$ The results again show that the benefits of BoN-DR are robust.}

In Appendix~\ref{app:validation_set}--\ref{app:efficiency}, we provide additional sensitivity analyses on the size of the validation set and the Monte Carlo pool. We study how the amount of validation data affects downstream policy value estimation and budget selection, and show that increasing the Monte Carlo pool reduces the additional Monte Carlo variance of BoN-DR and brings it closer to the efficiency bound under a correctly specified reward estimator.

\begin{figure}[htbp]
    \centering
    \includegraphics[width=0.80\linewidth]{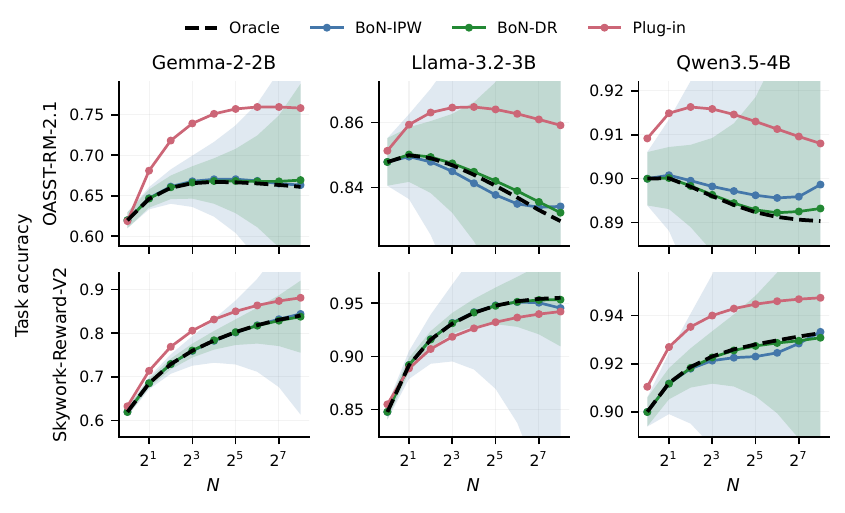}
    \caption{\textbf{BoN policy value estimation for realistic LLM setting on GSM8K.} Rows are external reward models used as score functions, and columns are reference models. \textbf{Black dashed:} oracle benchmark policy value. \textbf{Lines:} mean estimated policy value across repetitions. \textbf{Shaded:} $80\%$ pointwise CI. Plug-in CIs are not faithful wrt. nominal coverage and thus omitted.}
    \label{fig:ope-curves}
\end{figure}

\rqtwo \textbf{Performance in recovering BoN policy values in realistic LLM settings.} Figure~\ref{fig:ope-curves} compares the estimated policy values from Plug-in, BoN-IPW, and BoN-DR with the oracle policy values under different configurations. BoN-DR shows the most accurate estimates. Moreover, at $N=256$, BoN-DR yields substantially narrower CIs, with BoN-IPW intervals being $1.7$--$5.5\times$ wider. Notably, the Plug-in estimator remains biased even though the score function ranks responses accurately in the Skywork setting. This illustrates that effective ranking does not necessarily translate into accurate policy value estimation. \textit{$\Rightarrow$ BoN-DR accurately recovers BoN policy values in real-world settings while providing substantially more precise estimates than BoN-IPW.}

We provide additional results for CI coverage, error decomposition, overlap at large budgets, and robustness to alternative reward predictors in Appendices~\ref{app:real-coverage}--\ref{app:nuisance-robustness}.

\rqthree \textbf{Safe inference-time budget selection (``no-harm guarantee'').} We study whether Safe Improvement, which explicitly accounts for estimation uncertainty, helps avoid harmful choices of $N$, i.e., budgets that reduce policy value relative to the reference policy. This is particularly important under reward overoptimization with a strong reference model, where larger budgets achieve higher scores but can lower downstream performance. For comparison, we also report a Na{\"i}ve approach here (which chooses the response with the highest score at a fixed budget).

\begin{figure}[htbp]
    \centering
    \includegraphics[width=0.55\linewidth]{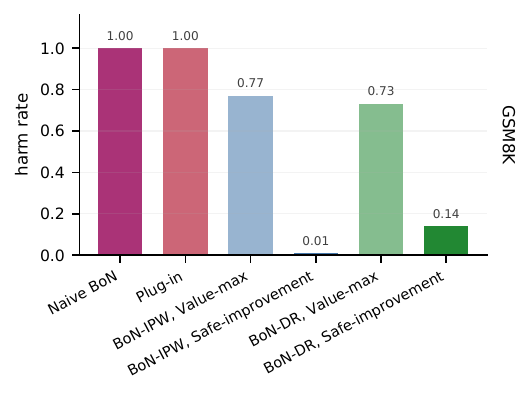}
    \caption{\textbf{Comparison of budget selection rules for Qwen3.5-4B with OASST-RM-2.1.} Here, the harm rate is defined as the fraction of evaluations with negative gain relative to $N=1$. Lower is better.
    Plug-in uses Value Maximization, and Na\"ive BoN selects $N_{\max}=256$.
    }
    \label{fig:selection}
\end{figure}

We have two main findings (see Figure~\ref{fig:selection}). First, both Na{\"i}ve BoN and Plug-in perform poorly overall and often lead to harmful selections. In contrast, our proposed selection rules (Value Maximization and Safe Improvement) are substantially better. Second, Safe Improvement leads to fewer harmful selections than Value Maximization. Appendix~\ref{app:selection-extra} reports detailed results across various configurations. \textit{$\Rightarrow$ Thanks to the uncertainty-aware approach, Safe Improvement can reduce harmful selections compared to Value Maximization, which is in line with our no-harm guarantee.}

\textbf{Conclusion.} We present, to our knowledge, the first framework for evaluating and selecting BoN policies under sample-only access, with valid inference and a no-harm guarantee. As a \textbf{limitation}, it requires logged data and loses precision at large budgets. Future work could use our per-prompt scores to learn prompt-specific budgets and extend likelihood-free evaluation to other inference-time methods built on black-box models.

\clearpage

\bibliographystyle{iclr2027_conference}
\bibliography{bibliography}

\clearpage

\appendix

\section{Extended related work}
\label{app:related_work}

\textbf{BoN and inference-time compute.}
Theoretical work characterizes the BoN policy, its KL divergence to the reference policy, and its optimality and scaling behavior \citep{gui2024bonbon, yang2024asymptotics, mroueh2025information, beirami2025guarantees, huang2025best, sriraman2026revisiting}. More broadly, repeated sampling with a verifier or reward model is a central way to scale inference-time compute \citep{brown2024large, snell2024scaling}, where best-of-$k$ performance is commonly estimated via the unbiased pass@$k$ estimator \citep{chen2021evaluating}. Our plug-in term (Eq.~\eqref{eq:plugin-u-stat}) follows the same U-statistic principle, but such estimators require the reward of every sampled response. In our setting, rewards are observed only for logged responses, which is why BoN-DR needs its correction term.

\textbf{Mitigating reward overoptimization.}
Existing mitigations change the deployed policy by modifying (i)~the score, e.g., via reward-model ensembles \citep{coste2024reward} or uncertainty penalties \citep{yu2026curiosity}; (ii)~the selection rule, e.g., via regularized, pessimistic, or soft BoN \citep{jinnai2025regularized, huang2025best, verdun2025soft, aminian2026best}; or (iii)~the model, via fine-tuning or distillation \citep{gui2024bonbon, sessa2025bond, chow2025inference}. Some of these additionally require response likelihoods or model weights. Our work is complementary: we keep standard BoN fixed and evaluate it and select its budget from logged data.

\textbf{Selecting inference-time parameters.}
Closest to our work, \citet{khalaf2025inferencetime} tune the BoN budget via the peak of the true reward. Their method is sample-only but requires true rewards for all sampled responses and provides no uncertainty quantification. \citet{raman2025adabon} instead allocate a given budget adaptively across prompts. In contrast, we select $N$ from logged data with simultaneous confidence bounds and a no-harm guarantee.

\textbf{Off-policy evaluation.}
Beyond the classical estimators \citep{li2011unbiased, horvitz1952generalization, strehl2010learning, Dudik.2011}, OPE has been extended to improve efficiency \citep{wang2017optimal, narita2019efficient, Kallus.2020}, handle limited overlap or distribution shift \citep{Kallus.2018, Kallus.2021b, Schweisthal.2023, Hatt.2022b, hess2026efficient}, and large action spaces via embeddings \citep{saito2022off}, as well as to human-feedback and preference data \citep{li2024policy, bhargava2024off, frauen2026nonparametric}. All of these require the action probabilities of the behavior policy, which are unavailable for free-form text under sample-only access. We instead exploit the known BoN selection mechanism, which expresses the density ratio through rank probabilities.

\textbf{Uncertainty-aware policy selection.}
High-confidence OPE and safe policy improvement use lower confidence bounds to avoid underperforming a baseline \citep{thomas2015highOPE, thomas2015highimprove, laroche2019safe}, with related confidence-based selection methods for contextual bandits \citep{kuzborskij2021confident, karampatziakis2021off, ryu2025improved, jin2025policy}. Safe Improvement applies this principle to BoN, with simultaneous bounds across budgets that share the same logged data and auxiliary samples.

\newpage

\section{Proofs}
\label{app:proofs}

\subsection{Regularity conditions}
\label{app:regularity}

We first make precise the regularity conditions used in the asymptotic
results. For the efficiency result, let $s_n$ and $\hat\mu_n$ denote
the fitted score and reward functions at sample size $n$; the main text
suppresses this dependence on $n$. These functions may coincide.

Let $\mathcal T_n$ denote the data used to fit $(s_n,\hat\mu_n)$.

\begin{assumption}[Regularity conditions]
\label{ass:regularity}
The following conditions hold.

\begin{enumerate}
    \item[(i)] \textbf{Sampling.}
    The evaluation observations
    \[
    Z_i=(X_i,A_i,Y_i),\qquad i=1,\ldots,n,
    \]
    are i.i.d., with
    \[
    A_i\mid X_i=x\sim\pi_b(\cdot\mid x),
    \qquad
    \E[Y_i\mid X_i=x,A_i=a]=\mu(x,a).
    \]
    The action space $\mathcal A$ is finite or countable.
    The training data $\mathcal T_n$ are independent of the
    evaluation sample.

    Conditional on $X_i$, the auxiliary actions
    \[
    \tilde A_{i1},\ldots,\tilde A_{iM}
    \overset{\mathrm{i.i.d.}}{\sim}\pi_b(\cdot\mid X_i)
    \]
    are independent of $(A_i,Y_i)$, and auxiliary pools are
    independent across $i$.

    \item[(ii)] \textbf{Budgets.}
    The BoN budget $N$ is fixed as $n\to\infty$.
    For the finite-pool result, $M\geq N-1$ is fixed.
    For the efficiency result, $M=M_n\to\infty$.

    \item[(iii)] \textbf{Moments.}
    For finite-$M$ validity,
    \[
    \E[Y^2]<\infty,
    \qquad
    \E[\hat\mu_n(X,A)^2]<\infty,
    \]
    where $A\sim\pi_b(\cdot\mid X)$.
    For the efficiency result, for some $\delta>0$,
    \[
    \E[|Y|^{2+\delta}]<\infty,
    \qquad
    \sup_n
    \E[
        |\hat\mu_n(X,A)|^{2+\delta}
    ]<\infty.
    \]

    \item[(iv)] \textbf{reward function consistency.}
    For the efficiency result,
    \[
    \|\hat\mu_n-\mu\|_{L_2(P_{X,A})}
    =o_p(1),
    \]
    where $P_{X,A}$ is induced by
    $X\sim P_X$ and $A\sim\pi_b(\cdot\mid X)$.

    \item[(v)] \textbf{Nondegenerate limiting variance.}
    Define
    \[
    \phi_{N,n}^{\mathrm{eff}}(Z)
    =
    \sum_{a\in\mathcal A}
    \pi_N^{s_n}(a\mid X)\mu(X,a)
    +
    w_N^{s_n}(X,A)
    \{Y-\mu(X,A)\}
    -
    V(\pi_N^{s_n}).
    \]
    We assume
    \[
    Var(
        \phi_{N,n}^{\mathrm{eff}}(Z)
    )
    \longrightarrow
    \sigma_{N,\mathrm{eff}}^2
    \in(0,\infty).
    \]
    The finite-$M$ variances appearing in
    Theorem~\ref{thm:dr-asymptotics}(a) are also assumed positive
    and finite.
\end{enumerate}
\end{assumption}

The efficiency statement treats the behavior sampler
$\pi_b(\cdot\mid x)$ and, conditional on the training sample, the
induced BoN policy as fixed. Thus the unknown components relevant to
the efficiency bound are the context distribution and conditional
reward distribution.

Unlike generic doubly robust estimators with two learned nuisance
functions, no product-rate condition such as
\[
\|\hat\mu-\mu\|_2\,
\|\hat w-w\|_2
=o_p(n^{-1/2})
\]
is required. The pooled weight estimator is conditionally unbiased
for the exact BoN density ratio, so finite auxiliary sampling
introduces mean-zero Monte Carlo variation rather than
nuisance-estimation bias. The independent training sample also avoids
Donsker or entropy conditions on the fitted model classes.

\subsection{Proof of Proposition~\ref{prop:bon-dr-score}}

\begin{proof}
Fix $x$ and $a$ with $\pi_b(a\mid x)>0$, and write
\[
F^-:=F_x^-(a),
\qquad
F:=F_x(a),
\qquad
q:=F-F^-.
\]
Here $q$ is the probability that an independent behavior-policy draw
has the same score as $a$, so
$q\geq\pi_b(a\mid x)>0$.

Condition on one of the $N$ BoN candidates being equal to $a$.
If exactly $t$ of the other $N-1$ candidates have equal score and
the remaining $N-1-t$ have smaller score, this occurrence of $a$ is
selected with probability $1/(t+1)$. By exchangeability of the
candidate positions,
\begin{align}
w_N^s(x,a)
&=
N\sum_{t=0}^{N-1}
\binom{N-1}{t}
q^t(F^-)^{N-1-t}
\frac{1}{t+1}
\\
&=
\sum_{k=1}^N
\binom Nk
q^{k-1}(F^-)^{N-k}
\\
&=
\frac{F^N-(F^-)^N}{F-F^-}.
\end{align}
Using
\[
\frac{u^N-v^N}{u-v}
=
\sum_{\ell=0}^{N-1}
u^{N-1-\ell}v^\ell
\]
gives
\[
w_N^s(x,a)
=
g_N(F_x(a),F_x^-(a)).
\]

This representation also implies
\begin{equation}
0\leq w_N^s(x,a)\leq N,
\label{eq:app-weight-bound}
\end{equation}
because $g_N$ is a sum of $N$ terms in $[0,1]$.

Finally,
\begin{align}
\sum_a
\pi_N^s(a\mid x)\hat\mu(x,a)
&=
\sum_a
\pi_b(a\mid x)
w_N^s(x,a)\hat\mu(x,a)
\\
&=
\E_{A'\sim\pi_b(\cdot\mid x)}
\left[
g_N(F_x(A'),F_x^-(A'))
\hat\mu(x,A')
\right],
\end{align}
where actions outside the support of $\pi_b$ have zero probability
under BoN. Substituting the same expression for the density ratio in
the correction term gives~\eqref{eq:bon-dr-score}.
\end{proof}

\subsection{Finite-pool unbiasedness}

For a function $f$, define
\[
m_{N,f}(x)
:=
\E_{A\sim\pi_N^s(\cdot\mid x)}
[f(x,A)].
\]

\begin{lemma}
\label{lem:pooled-unbiased}
For every $M\geq N-1$,
\begin{align}
\E[
\hat v_{i,N}^{s,M+1}
\mid X_i
]
&=
m_{N,\hat\mu}(X_i),
\label{eq:app-plugin-unbiased}
\\
\E[
\hat w_{i,N}^{s,M}
\mid X_i,A_i
]
&=
w_N^s(X_i,A_i).
\label{eq:app-weight-unbiased}
\end{align}
Moreover,
\[
0\leq
\hat w_{i,N}^{s,M}
\leq N.
\]
\end{lemma}

\begin{proof}
Conditional on $X_i=x$, the stacked actions
\[
A_i,\tilde A_{i1},\ldots,\tilde A_{iM}
\]
are $M+1$ independent draws from $\pi_b(\cdot\mid x)$. Hence every
size-$N$ subset in~\eqref{eq:plugin-u-stat} has the distribution of
$N$ fresh BoN candidates. The kernel $h_{i,J}$ analytically averages
over uniform tie-breaking, so
\[
\E[h_{i,J}\mid X_i=x]
=
m_{N,\hat\mu}(x).
\]
Averaging over subsets proves~\eqref{eq:app-plugin-unbiased}.

For completeness, an $N$-subset has maximum score $r_{ig}$ exactly
when it is contained among the $E_{ig}$ observations with score at
most $r_{ig}$ but not among the $L_{ig}$ observations with strictly
smaller score. There are therefore
\[
\binom{E_{ig}}N-\binom{L_{ig}}N
\]
such subsets. Symmetry across the tied maximizing occurrences gives
their average prediction $\bar\mu_{ig}$, yielding
\eqref{eq:plugin-pooled}.

For the weight estimator, fix $(X_i,A_i)=(x,a)$ and a subset $J$ of
$N-1$ auxiliary observations. Define
\[
Q_J
=
N
\frac{
\1\{
s(x,\tilde A_{ij})\leq s(x,a)
\text{ for all }j\in J
\}
}{
1+
\sum_{j\in J}
\1\{
s(x,\tilde A_{ij})=s(x,a)
\}
}.
\]
The quantity $Q_J/N$ is exactly the probability that the distinguished
logged occurrence $a$ is selected when it is combined with the
$N-1$ candidates in $J$. The argument in
Proposition~\ref{prop:bon-dr-score} therefore gives
\[
\E[Q_J\mid X_i=x,A_i=a]
=
w_N^s(x,a).
\]
Averaging over all such subsets proves
\eqref{eq:app-weight-unbiased}. Grouping subsets according to the
number $t$ of equal-score competitors gives exactly
\eqref{eq:weight-pooled}. Since $0\leq Q_J\leq N$, the same bound
holds for $\hat w_{i,N}^{s,M}$.
\end{proof}

\begin{lemma}
\label{lem:score-unbiased}
For every $M\geq N-1$,
\[
\E[
\psi_N^M(Z_i,\tilde A_i)
]
=
V(\pi_N^s),
\]
for any fixed reward function satisfying the required moment
conditions.
\end{lemma}

\begin{proof}
Lemma~\ref{lem:pooled-unbiased} gives
\[
\E[
\hat v_{i,N}^{s,M+1}
\mid X_i
]
=
m_{N,\hat\mu}(X_i).
\]
Also, conditional on $(X_i,A_i)$, the auxiliary pool is independent
of $Y_i$, so
\begin{align}
&
\E\left[
\hat w_{i,N}^{s,M}
\{Y_i-\hat\mu(X_i,A_i)\}
\mid X_i,A_i
\right]
\\
&\qquad=
w_N^s(X_i,A_i)
\{\mu(X_i,A_i)-\hat\mu(X_i,A_i)\}.
\end{align}
Averaging over $A_i\sim\pi_b(\cdot\mid X_i)$ gives
\[
m_{N,\mu}(X_i)-m_{N,\hat\mu}(X_i).
\]
Thus
\[
\E[
\psi_N^M(Z_i,\tilde A_i)
\mid X_i
]
=
m_{N,\mu}(X_i),
\]
and taking expectation over $X_i$ proves the result.
\end{proof}

\subsection{Proof of Theorem~\ref{thm:dr-asymptotics}}

\begin{proof}
For fixed $M$, conditional on the fitted functions,
$(Z_i,\tilde A_i)$ are i.i.d. By
Lemma~\ref{lem:score-unbiased},
\[
\E[
\psi_N^M(Z_i,\tilde A_i)
]
=
V(\pi_N^s).
\]
The finite-second-moment assumption and the ordinary central limit
theorem therefore give
\[
\sqrt n
\left(
\hat V_{\mathrm{DR}}^M(\pi_N^s)
-
V(\pi_N^s)
\right)
\rightsquigarrow
\mathcal N(0,\sigma_{N,M}^2),
\]
proving part~(a). No consistency of $\hat\mu$ is used.

For part~(b), define the exact-ratio DR score
\[
\psi_{N,n}^{\infty}(Z)
=
m_{N,\hat\mu_n}(X)
+
w_N^{s_n}(X,A)
\{Y-\hat\mu_n(X,A)\}.
\]
A fixed-order U-statistic with a square-integrable kernel converges in
$L_2$ to its expectation as the pool size increases; more precisely,
its variance is $O(M^{-1})$ for fixed kernel order. Applying this to
the plug-in and weight kernels, using
$0\leq\hat w_N^{s_n,M},w_N^{s_n}\leq N$, yields
\begin{equation}
\left\|
\psi_N^{M_n}(Z,\tilde A;\hat\mu_n)
-
\psi_{N,n}^{\infty}(Z)
\right\|_2
=o_p(1).
\label{eq:app-pool-l2}
\end{equation}
The moment conditions in Assumption~\ref{ass:regularity} ensure the
corresponding U-statistic second moments are uniformly bounded.

Let
\[
\delta_n(x,a)
=
\hat\mu_n(x,a)-\mu(x,a).
\]
Since $w_N^{s_n}\leq N$,
\begin{align}
\E[
m_{N,\delta_n}(X)^2
]
&\leq
\E_{X,A\sim\pi_N^{s_n}}
[
\delta_n(X,A)^2
]
\\
&=
\E_{X,A\sim\pi_b}
[
w_N^{s_n}(X,A)
\delta_n(X,A)^2
]
\\
&\leq
N
\|\delta_n\|_2^2.
\end{align}
Hence
\begin{align}
&
\left\|
\psi_{N,n}^{\infty}(Z)
-
\left[
m_{N,\mu}(X)
+
w_N^{s_n}(X,A)
\{Y-\mu(X,A)\}
\right]
\right\|_2
\\
&\qquad\leq
(\sqrt N+N)
\|\hat\mu_n-\mu\|_2
=o_p(1).
\label{eq:app-mu-l2}
\end{align}
Combining~\eqref{eq:app-pool-l2} and
\eqref{eq:app-mu-l2},
\[
\left\|
\psi_N^{M_n}(Z,\tilde A;\hat\mu_n)
-
\left[
V(\pi_N^{s_n})
+
\phi_{N,n}^{\mathrm{eff}}(Z)
\right]
\right\|_2
=o_p(1).
\]

Both scores have expectation $V(\pi_N^{s_n})$. Therefore, if $R_{i,n}$
denotes their difference,
\[
\E[R_{i,n}]=0,
\qquad
\E[R_{i,n}^2]=o_p(1),
\]
and independence across evaluation observations gives
\[
\frac1{\sqrt n}
\sum_{i=1}^n R_{i,n}
=o_p(1).
\]
Consequently,
\begin{equation}
\sqrt n
\left(
\hat V_{\mathrm{DR}}^{M_n}(\pi_N^{s_n})
-
V(\pi_N^{s_n})
\right)
=
\frac1{\sqrt n}
\sum_{i=1}^n
\phi_{N,n}^{\mathrm{eff}}(Z_i)
+
o_p(1).
\label{eq:app-asym-linear}
\end{equation}
The $(2+\delta)$ moment condition implies the Lindeberg condition, and
the assumed variance convergence therefore yields
\[
\frac1{\sqrt n}
\sum_{i=1}^n
\phi_{N,n}^{\mathrm{eff}}(Z_i)
\rightsquigarrow
\mathcal N(
0,\sigma_{N,\mathrm{eff}}^2
).
\]
Slutsky's theorem proves part~(b).

Finally, conditional on the behavior sampler and fitted score,
$\pi_N^{s_n}$ is a fixed target policy. The standard AIPW influence
function for its value is
\[
m_{N,\mu}(X)
-
V(\pi_N^{s_n})
+
w_N^{s_n}(X,A)
\{Y-\mu(X,A)\}.
\]
Indeed, perturbing the context distribution produces the first term,
while perturbing the conditional reward distribution produces the
second; because $\pi_b$ is treated as fixed, there is no additional
behavior-policy term. Thus
$\phi_{N,n}^{\mathrm{eff}}$ is the efficient influence function in
this model, and~\eqref{eq:app-asym-linear} establishes the claimed
efficiency.

The argument also shows why no rate condition between $M_n$ and $n$
is required: the finite-pool approximation is mean zero for every
$M_n$, so it is enough that its per-observation variance vanish as
$M_n\to\infty$.
\end{proof}

\subsection{Proof of Corollary~\ref{thm:value-maximization}}

\begin{proof}
Write
\[
V_N=V(\pi_N^s),
\qquad
\hat V_N=\hat V_{\mathrm{DR}}^M(\pi_N^s).
\]
By Theorem~\ref{thm:dr-asymptotics}(a) and finiteness of
$\mathcal N$,
\[
\epsilon_n
:=
\max_{N\in\mathcal N}
|\hat V_N-V_N|
=o_p(1).
\]
Let $N^\star\in\argmax_{N\in\mathcal N}V_N$. Since
$\hat N_{\mathrm{value}}$ maximizes $\hat V_N$,
\[
V_{N^\star}
\leq
\hat V_{N^\star}+\epsilon_n
\leq
\hat V_{\hat N_{\mathrm{value}}}+\epsilon_n
\leq
V_{\hat N_{\mathrm{value}}}+2\epsilon_n.
\]
Therefore
\[
0
\leq
\max_{N\in\mathcal N}V_N
-
V_{\hat N_{\mathrm{value}}}
\leq
2\epsilon_n
=o_p(1),
\]
which proves the result. Uniqueness of the maximizing budget is not
required because the conclusion concerns policy value rather than the
selected index.
\end{proof}

\subsection{Proof of Theorem~\ref{thm:safe-improvement}}

\begin{proof}
Let
\[
\mathcal N_+
=
\mathcal N\setminus\{1\},
\]
and collect the improvement scores into
\[
D_i
=
(D_{i,N}^M:N\in\mathcal N_+),
\qquad
\Delta
=
(\Delta_N:N\in\mathcal N_+).
\]
By Lemma~\ref{lem:score-unbiased},
\[
\E[D_{i,N}^M]=\Delta_N.
\]
Since $\mathcal N$ is fixed and finite, the multivariate central limit
theorem gives
\[
\sqrt n(\hat\Delta^M-\Delta)
\rightsquigarrow
\mathcal N(0,\Sigma),
\qquad
\Sigma=Var(D_i),
\]
and the empirical covariance matrix converges in probability to
$\Sigma$.

Conditional on the data,
\[
\frac1{\sqrt n}
\sum_{i=1}^n
\xi_i
(D_i-\hat\Delta^M)
\]
is exactly Gaussian with the empirical covariance matrix. Hence the
multiplier bootstrap consistently estimates the distribution of the
maximum standardized estimation error. Therefore
\begin{equation}
\prob\left(
\Delta_N
\geq
\widehat{\mathrm{LCB}}_\Delta(N)
\text{ for every }N\in\mathcal N_+
\right)
=
1-\alpha+o(1).
\label{eq:app-simultaneous}
\end{equation}

On the event in~\eqref{eq:app-simultaneous}, if
$\hat N_{\mathrm{LCB}}=1$ then the selected value equals
$V(\pi_b)$. Otherwise,
\[
\widehat{\mathrm{LCB}}_\Delta(
\hat N_{\mathrm{LCB}}
)>0,
\]
so simultaneous coverage implies
\[
\Delta_{\hat N_{\mathrm{LCB}}}>0.
\]
Thus
\[
\prob\left(
V(\pi_{\hat N_{\mathrm{LCB}}}^s)
\geq
V(\pi_b)
\right)
\geq
1-\alpha-o(1).
\]

For oracle consistency, finiteness of the grid gives
\[
\max_{N\in\mathcal N_+}
|\hat\Delta_N^M-\Delta_N|
=o_p(1).
\]
The bootstrap critical value and estimated standard deviations are
$O_p(1)$, so the confidence penalties are $o_p(1)$ and therefore
\[
\max_{N\in\mathcal N_+}
\left|
\widehat{\mathrm{LCB}}_\Delta(N)
-
\Delta_N
\right|
=o_p(1).
\]
Let
\[
\Delta^\star
=
\max_{N\in\mathcal N_+}\Delta_N.
\]
If $\Delta^\star>0$, the positivity threshold is eventually passed
and the finite-grid argmax argument implies that the selected
improvement converges to $\Delta^\star$. If $\Delta^\star<0$, all
lower bounds are eventually negative and the procedure selects
$N=1$. If $\Delta^\star=0$, either the baseline is selected or a
budget whose population improvement converges to zero is selected.
In all cases,
\[
V(\pi_{\hat N_{\mathrm{LCB}}}^s)
\overset{p}{\longrightarrow}
V(\pi_b)+\max\{0,\Delta^\star\}
=
\max_{N\in\mathcal N}
V(\pi_N^s).
\]
\end{proof}

\clearpage

\section{Experimental details}
\label{app:experimental-details}

\subsection{Synthetic experiments}
\label{app:controlled-details}

\textbf{Generative model.} Let $X$ be uniform on $\{-1,+1\}$. The action space contains $16$ main actions and eight tail actions,
\[
\mathcal A=\mathcal A_{\mathrm{main}}\cup\mathcal A_{\mathrm{tail}}.
\]
Their base qualities are
\[
q_0(a_j)=0.55+\frac{0.30(j-1)}{15},\quad j=1,\ldots,16,
\qquad
q_0(b_k)=0.05+\frac{0.20(k-1)}{7},\quad k=1,\ldots,8.
\]
The behavior policy is independent of context:
\[
\pi_b(a\mid x)=
\begin{cases}
0.96/16,&a\in\mathcal A_{\mathrm{main}},\\
0.04/8,&a\in\mathcal A_{\mathrm{tail}}.
\end{cases}
\]
Conditional rewards satisfy
\[
Y\mid X=x,A=a\sim\operatorname{Bernoulli}(\mu(x,a)),
\qquad \mu(x,a)=q_0(a)+0.04x.
\]

\textbf{Independent error controls.} We specify the score function and reward function as
\[
s_\gamma(x,a)=\mu(x,a)+\gamma\,\mathbf1\{a\in\mathcal A_{\mathrm{tail}}\},
\]
\[
\hat\mu_\delta(x,a)=\operatorname{clip}\!\left(
\mu(x,a)+\delta\,\mathbf1\{a\in\mathcal A_{\mathrm{tail}}\},0,1\right).
\]
Thus, $\gamma$ changes the ranking of tail actions and hence the target policy $\pi_N^{s_\gamma}$. In contrast, $\delta$ changes the reward function predictions while leaving the target policy fixed. At $\delta=0$, the reward function is correct. The default configuration is $\gamma=0.65$, $\delta=0.4$. The selection comparison in Figure~\ref{fig:selection} uses $\delta=0.6$, and the efficiency experiment in Appendix~\ref{app:efficiency} uses $\delta=0$.

\textbf{Sampling and evaluation.} Each repetition contains $n$ independent logged observations and $M$ independent auxiliary behavior-policy draws per logged context. The main experiments use $n=4{,}096$, $300$ repetitions, and the candidate grid $\mathcal N=\{1,2,4,\ldots,256\}$, with $M\geq N_{\max}$, satisfying the shared-pool requirement $M\geq N_{\max}-1$. Sample-size and auxiliary-pool sweeps vary the corresponding quantity. Shared random streams pair the data across estimators and error settings.

We compute the population policy value exactly:
\[
V_N(\gamma)=\frac12\sum_{x\in\{-1,+1\}}\sum_{a\in\mathcal A}
\pi_N^{s_\gamma}(a\mid x)\mu(x,a).
\]
The oracle budget is the smallest maximizer of $V_N(\gamma)$ over $\mathcal N$. Population values and infinite-pool variance benchmarks are used only for evaluation. The estimators receive sampled observations, score evaluations, and reward function predictions. Experimental regimes were fixed using population quantities before running the repetitions.

\subsection{GSM8K environment}
\label{app:gsm8k-details}

GSM8K provides $7{,}473$ training and $1{,}319$ test questions. We split the training questions deterministically into $3{,}500$ prompts for fitting the reward function and $3{,}973$ validation prompts for off-policy evaluation and selection, and reserve the official test set for the exact evaluation of selected policies. For each prompt we store $512$ independently sampled responses with their exact-match rewards and their scores under both reward models. We use $\mathcal N=\{1,2,4,\ldots,256\}$ as the candidate grid, with$ M\ge N_{\max}$.

Each validation repetition reveals the reward of one uniformly drawn response per prompt; the remaining responses of that prompt form the auxiliary pool used to estimate rank probabilities. This mimics the logged-data setting in which one response per prompt received feedback. We use $500$ logging repetitions for estimation and pointwise inference, and $100$ repetitions for policy selection.

The reward function is a cubic-spline calibration of the reward model score with six quantile knots, a logistic link, fixed $L_2$ regularization, and constant extrapolation beyond the observed score range. It is fit once on the $3{,}500$ training prompts and held fixed across repetitions, budgets, and selection rules. This is a deliberately favorable choice for the Plug-in estimator: the reward function is a calibration of the same score that defines the target policy, so any bias we report is not caused by an unrelated or weak nuisance. Appendix~\ref{app:nuisance-robustness} confirms the conclusion with two independently trained reward functions.

\subsection{Baseline estimators}
\label{app:baseline_estimators}

We consider two baseline estimators corresponding to the two components of the pooled BoN-DR estimator in Eq.~\eqref{eq:pooled-dr-estimator}: BoN-IPW and Plug-in. The Plug-in estimator retains only the plug-in term, whereas BoN-IPW uses the estimated BoN density ratio appearing in the correction term.

\textbf{Pooled BoN-IPW estimator.}
BoN-IPW uses the estimated BoN density ratio from the correction term of Eq.~\eqref{eq:pooled-dr-estimator}. We define the pooled BoN-IPW estimator as
\begin{align}
\hat V_{\mathrm{IPW}}^M(\pi_N^s) &= \frac{1}{n} \sum_{i=1}^n \hat w_{i,N}^{s,M}(x_i, \bar a_i)  y_i \label{eq:ipw-estimator}
\end{align}
Thus, BoN-IPW relies on the logged rewards together with the estimated BoN density ratio, without using the plug-in term of BoN-DR.

\textbf{Plug-in estimator.}
The Plug-in estimator corresponds directly to the plug-in term of the pooled BoN-DR estimator in Eq.~\eqref{eq:pooled-dr-estimator}. We define the Plug-in estimator as
\begin{align}
\hat V_{\mathrm{Plug-in}}^M(\pi_N^s) &= \frac{1}{n} \sum_{i=1}^n  \hat v_{i,N}^{s,M+1}(x_i, \bar a_i) \label{eq:plug-in-estimator}
\end{align}
The Plug-in estimator therefore evaluates the target BoN policy using the estimated reward function alone, without the correction based on the logged reward. Consequently, errors in the estimated reward function are not corrected by the Plug-in estimator.

\subsection{Confidence statements}
\label{app:inference}

For DR, $\hat\sigma_{N,M}$ is the empirical standard deviation of the scores $\psi_N^M(z_i,\tilde a_i)$; the other estimators use their corresponding scores. Pointwise confidence intervals use the Wald form $\hat V_{\mathrm{DR}}^M(\pi_N^s)\pm z_{1-\alpha/2}\hat\sigma_{N,M}/\sqrt{n}$. Simultaneous bounds over a candidate grid use the Gaussian multiplier bootstrap of Section~\ref{sec:select_N} with $2{,}000$ draws and $\alpha=0.05$, with the critical value recomputed for each grid prefix rather than reused from the full grid. Multiplier draws are shared across budgets so that the dependence between value estimates is preserved.

\subsection{Metric definitions}
\label{app:metrics}

Bias is the mean estimate minus the exact value; RMSE is the square root of the mean squared error; the empirical dispersion uses $\mathrm{ddof}=1$; interval width is the upper minus the lower limit. Coverage is the fraction of repetitions whose interval contains the exact value, counted once per repetition. Simultaneous coverage counts one event per repetition and grid, never one event per budget.

For selection, gain is the exact value of the selected policy minus the exact value at $N=1$. On GSM8K, estimation and coverage are evaluated against the exact benchmark induced by the stored validation responses, while selected budgets are evaluated independently on the held-out test set. The reported GSM8K harm rate is therefore an empirical analogue of the theoretical no-harm criterion. In the Synthetic experiments, selected policies are evaluated on the exact population. The harm rate is the fraction of repetitions whose gain is negative beyond a numerical tolerance of $10^{-12}$; it is not the fraction of individually harmed prompts. Exact ties in a selection objective resolve to the smallest budget. The Safe Improvement rule includes a zero bound at $N=1$, so it returns $N=1$ whenever no budget has a strictly positive improvement bound. Monte Carlo uncertainty is binomial for coverage and harm, and the between-repetition standard deviation divided by the square root of the number of repetitions for bias and gain. To keep the tables readable these standard errors are not printed; they are retained for every reported cell in the result files described in Appendix~\ref{app:reproducibility}. At $300$ repetitions a coverage of $0.95$ has a Monte Carlo standard error of $0.013$, and at $100$ repetitions a harm rate of $0.14$ has one of $0.035$.

\subsection{Reproducibility}
\label{app:reproducibility}

All numbers in this paper are computed from a single frozen result set, which is fixed before any figure or table is produced. It contains $5{,}100$ per-repetition records for the Synthetic experiments and $3{,}600$ for GSM8K, covering every combination of environment, estimator, budget, and selection rule reported here and in the appendix. Random number generation uses named streams seeded per repetition, so contexts, logged actions, rewards, and auxiliary draws are shared across estimators and across misspecification settings. Any two estimators are therefore compared on identical data, and differences between them are not attributable to sampling variation in the inputs.

\section{Additional results for off-policy evaluation}
\label{app:estimation-extra}

\subsection{Reward overoptimization under score function misspecification}
\label{app:syn_exp_score_func}

We further study how misspecification of the score function affects the induced BoN policy. In the Synthetic experiments, the parameter $\gamma$ increases the scores of a small set of low-quality actions that are rarely sampled under the behavior policy. As $\gamma$ increases, these actions are more likely to be selected by BoN, particularly at larger sampling budgets $N$. Consequently, the induced BoN policy places increasing mass on actions with low probability under the behavior policy, leading to weaker overlap and larger density ratios. Figure~\ref{fig:syn-regimes} illustrates the resulting reward overoptimization: while the expected selected score continues to increase with $N$, the true policy value can peak and subsequently decline when the score function is sufficiently misspecified. Throughout this experiment, we fix the reward function error at $\delta=0.4$ and vary only the score-function error $\gamma$.

\begin{figure}[H]
    \centering
    \vspace{-0.4cm}
    \includegraphics[width=0.9\linewidth]{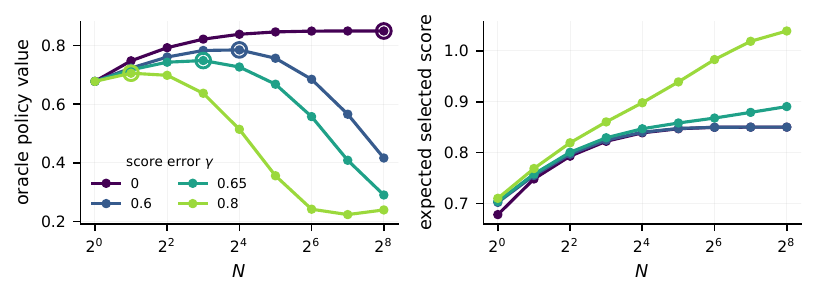}
    \vspace{-0.4cm}
    \caption{\textbf{Robust to score function misspecification} ($\delta = 0.4$). \textbf{Left:} exact policy value $V(\pi_N^s)$; circles show the optimal budget $N^\star$. \textbf{Right:} expected score of the selected action. Colors show the score error $\gamma$. }
    \label{fig:syn-regimes}
\end{figure}

\subsection{Sensitivity to validation-set size}
\label{app:validation_set}

To assess the sensitivity of downstream policy evaluation and budget selection to the amount of validation data, we vary the size of the logged validation set while keeping the remaining experimental setup fixed. Figure~\ref{fig:consistency} reports policy-value estimation at $N=256$ and Value Maximization selection over the candidate budget grid. As the validation size increases, the RMSE of BoN-DR and BoN-IPW decreases and their confidence-interval coverage approaches the nominal $95\%$ level, whereas the bias of the Plug-in estimator persists. The improved estimation accuracy also translates into more reliable budget selection: the probability of selecting the oracle-optimal budget $N^\star$ increases with the validation size, with BoN-DR reaching high selection accuracy more quickly than BoN-IPW. Overall, the results show that the downstream evaluation and selection procedure becomes increasingly stable as more logged validation data are available.

\begin{figure}[H]
    \centering
    \vspace{-0.6cm}
    \includegraphics[width=\linewidth]{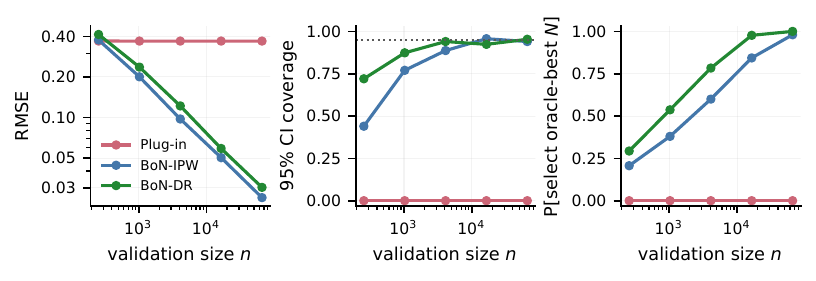}
    \vspace{-0.8cm}
    \caption{
    \textbf{
    Estimated BoN policy value at $N = 256$ as logged sample size grows.}
    \textbf{(a)} RMSE at log scale.
    \textbf{(b)} Coverage of the nominal $95\%$ pointwise confidence interval. The dotted line marks the nominal level.
    \textbf{(c)} Probability of selecting the optimal budget $N^\star$ under Value Maximization selection.}
    \label{fig:consistency}
\end{figure}

\subsection{Finite-$M$ efficiency}
\label{app:efficiency}

Figure~\ref{fig:efficiency} isolates the cost of the finite auxiliary pool in the Synthetic experiments with a correctly specified reward function ($\delta=0$). The paired ratio between the finite-$M$ and the infinite-$M$ DR variance falls from $1.335$ at $M=256$ to $1.019$ at $M=4096$ at $N=256$, so the Monte Carlo cost of the pool vanishes as $M$ grows, as Theorem~\ref{thm:dr-asymptotics} (b) predicts. The right panel compares the realized variance with the semiparametric efficiency bound: DR attains the bound closely, whereas IPW exceeds it by up to a factor of $4.9$.

\begin{figure}[H]
    \centering
    \includegraphics[width=\textwidth]{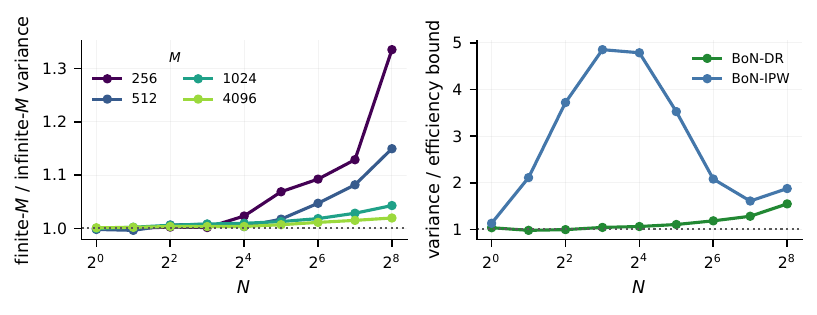}
    \caption{\textbf{Finite-$M$ cost and efficiency in the synthetic experiments.} Correctly specified reward function ($\delta=0$), $\gamma=0.65$, $n=4{,}096$, $300$ repetitions. \textbf{Left:} ratio of the finite-$M$ DR variance to the paired infinite-$M$ oracle variance. \textbf{Right:} ratio of the realized variance to the efficiency bound at $M=256$. The dotted line marks a ratio of one.}
    \label{fig:efficiency}
\end{figure}

\subsection{Real coverage}
\label{app:real-coverage}

Figure~\ref{fig:real-coverage} examines interval calibration over the full budget grid on GSM8K. Across both OASST and Skywork, BoN-DR and BoN-IPW achieve coverage close to the nominal $95\%$ level over most budgets, despite the different BoN policies induced by the two score functions. BoN-DR coverage deteriorates at the largest budgets, where rank-weight overlap becomes weak, while BoN-IPW remains better calibrated at these endpoints. This undercoverage is consistent with a finite-sample limitation under increasingly extreme BoN reweighting. In contrast, Plug-in coverage drops to zero from $N=4$ onward in all six configurations. Overall, the results suggest that BoN-DR and BoN-IPW provide reasonably well-calibrated uncertainty quantification across both score-function regimes, with the main deviations occurring at the largest budgets.

\begin{figure}[H]
    \centering
    \includegraphics[width=\linewidth]{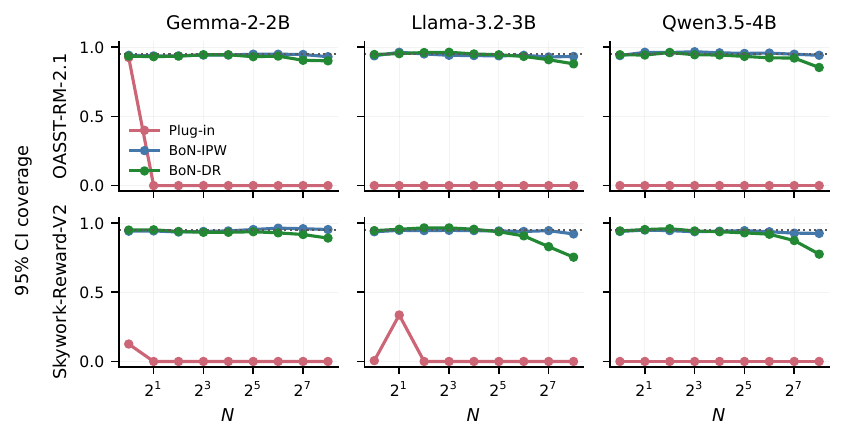}
    \vspace{-0.8cm}
    \caption{\textbf{Interval validity across budgets on GSM8K.} Empirical coverage of the nominal $95\%$ pointwise confidence interval over $500$ repetitions. The dotted line marks the nominal level. Rows are external reward models used as score functions, and columns are reference models. The Plug-in estimator never covers from $N=4$ on. DR declines at the largest budgets, where overlap deteriorates.}
    \label{fig:real-coverage}
\end{figure}

\subsection{Error decomposition over the budget grid}
\label{app:error-decomposition}

Figure~\ref{fig:bias-variance} decomposes the estimation error on GSM8K across all budgets, which a single-budget table cannot show. The Plug-in bias is large and nearly constant in $N$, whereas DR and IPW are essentially unbiased at every budget and their dispersion grows as overlap deteriorates. The Plug-in estimator therefore becomes competitive in RMSE only at the largest budgets, and only by trading a constant bias against negligible variance. Table~\ref{tab:real-estimation} gives the per-configuration numbers at $N=256$.

\begin{figure}[H]
    \centering
    \includegraphics[width=\textwidth]{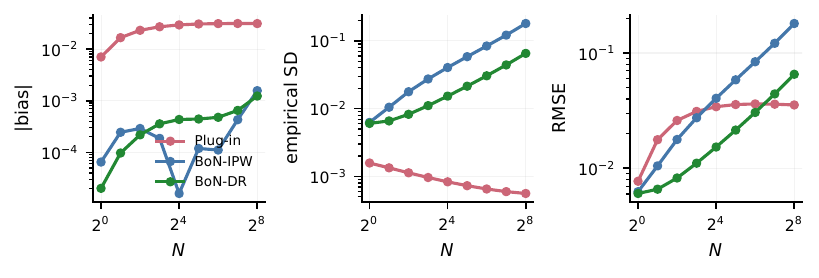}
    \caption{\textbf{Absolute bias, dispersion, and RMSE on GSM8K over the budget grid.} Each quantity is averaged over the six configurations, $500$ repetitions, log scale.}
    \label{fig:bias-variance}
\end{figure}

\begin{table}[H]
    \centering
    \caption{\textbf{Estimation on GSM8K at $N=256$.} $500$ repetitions with $n=3{,}973$ logged prompts. Coverage is that of the nominal $95\%$ pointwise confidence interval. \textbf{Bold} marks the lowest RMSE and the coverage closest to $0.95$ within each configuration.}
    \label{tab:real-estimation}
    \small
    \begin{tabular}{llrrrr}
\toprule
Reward model & Estimator & Bias & RMSE & CI width & Coverage \\
\midrule
\multicolumn{6}{l}{\emph{OASST-RM-2.1, reference model Gemma-2-2B}} \\
 & Plug-in & 0.097 & \textbf{0.097} & 0.002 & 0.000 \\
 & BoN-IPW & 0.002 & 0.156 & 0.625 & \textbf{0.932} \\
 & BoN-DR & 0.008 & 0.102 & 0.364 & 0.902 \\
\midrule
\multicolumn{6}{l}{\emph{OASST-RM-2.1, reference model Llama-3.2-3B}} \\
 & Plug-in & 0.029 & \textbf{0.029} & 0.001 & 0.000 \\
 & BoN-IPW & 0.005 & 0.179 & 0.703 & \textbf{0.934} \\
 & BoN-DR & 0.003 & 0.079 & 0.280 & 0.880 \\
\midrule
\multicolumn{6}{l}{\emph{OASST-RM-2.1, reference model Qwen3.5-4B}} \\
 & Plug-in & 0.018 & \textbf{0.018} & 0.001 & 0.000 \\
 & BoN-IPW & 0.008 & 0.187 & 0.726 & \textbf{0.942} \\
 & BoN-DR & 0.003 & 0.058 & 0.225 & 0.854 \\
\midrule
\multicolumn{6}{l}{\emph{Skywork-Reward-V2, reference model Gemma-2-2B}} \\
 & Plug-in & 0.041 & \textbf{0.041} & 0.007 & 0.000 \\
 & BoN-IPW & 0.004 & 0.172 & 0.704 & \textbf{0.954} \\
 & BoN-DR & -0.003 & 0.064 & 0.252 & 0.892 \\
\midrule
\multicolumn{6}{l}{\emph{Skywork-Reward-V2, reference model Llama-3.2-3B}} \\
 & Plug-in & -0.013 & \textbf{0.013} & 0.002 & 0.000 \\
 & BoN-IPW & -0.010 & 0.198 & 0.737 & \textbf{0.922} \\
 & BoN-DR & -0.002 & 0.040 & 0.134 & 0.754 \\
\midrule
\multicolumn{6}{l}{\emph{Skywork-Reward-V2, reference model Qwen3.5-4B}} \\
 & Plug-in & 0.015 & \textbf{0.015} & 0.001 & 0.000 \\
 & BoN-IPW & 0.001 & 0.190 & 0.727 & \textbf{0.926} \\
 & BoN-DR & -0.002 & 0.048 & 0.168 & 0.776 \\
\bottomrule
\end{tabular}

\end{table}

\subsection{Overlap diagnostics}
\label{app:overlap}
\label{app:stress-test}

Figure~\ref{fig:overlap} reports the effective sample size of the rank weights on GSM8K. It falls from $3{,}973$ prompts at $N=1$ to about $25$ at $N=256$ and about $10$ at $N=512$, where only $0.5\%$ of prompts carry non-zero weight. The endpoint $N=512$ uses all $M+1=512$ stored responses and satisfies the boundary condition $M=N-1$ in Theorem~\ref{thm:dr-asymptotics}. We report it only as an overlap stress test and restrict the main text to $N\leq256$. Extending the grid to $N_{\max}=512$ roughly doubles the residual harm of DR under Safe Improvement selection, from $0.12$ to $0.24$ on Llama with OASST and from $0.14$ to $0.33$ on Qwen with OASST, while IPW is unaffected at $0.00$ and $0.01$.

\begin{figure}[H]
    \centering
    \includegraphics[width=\textwidth]{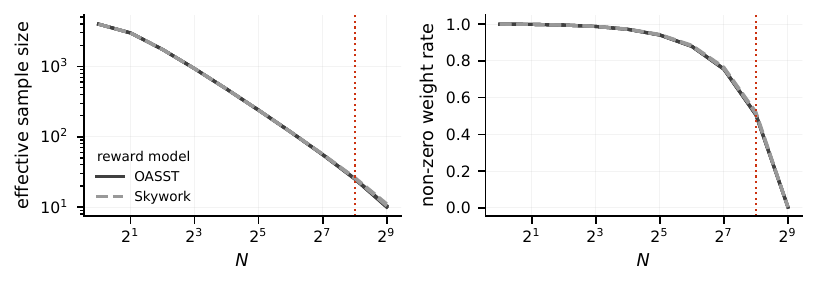}
    \caption{\textbf{Overlap of the rank weights on GSM8K.} \textbf{Left:} effective sample size, log scale, out of $3{,}973$ validation prompts. \textbf{Right:} fraction of prompts receiving non-zero weight. Solid lines are OASST, dashed lines are Skywork; the vertical line marks the largest budget used in the main text.}
    \label{fig:overlap}
\end{figure}

\subsection{Robustness to the reward function}
\label{app:nuisance-robustness}

Table~\ref{tab:nuisance-robustness} repeats the GSM8K comparison with two reward functions trained independently of the score function, a TF--IDF logistic regression and an embedding-based LightGBM model. Plug-in coverage is zero for both, so the failure reported in \rqtwo is a property of the estimator under misspecification and not an artifact of the score-calibration spline used in the main text.

\begin{table}[H]
    \centering
    \caption{\textbf{Estimation on GSM8K at $N=256$ with independently trained reward functions.} OASST score function. Coverage is that of the nominal $95\%$ pointwise confidence interval. \textbf{Bold} marks the lowest RMSE within each reference model and reward function.}
    \label{tab:nuisance-robustness}
    \small
    \begin{tabular}{llrrr}
\toprule
Reward estimator & Estimator & Bias & RMSE & Coverage \\
\midrule
\multicolumn{5}{l}{\emph{Reference model Gemma-2-2B}} \\
Embedding LightGBM & Plug-in & 0.040 & \textbf{0.040} & 0.000 \\
 & BoN-IPW & 0.008 & 0.167 & 0.880 \\
 & BoN-DR & -0.015 & 0.093 & 0.970 \\
TF--IDF logistic & Plug-in & 0.022 & \textbf{0.022} & 0.000 \\
 & BoN-IPW & 0.008 & 0.167 & 0.880 \\
 & BoN-DR & -0.011 & 0.101 & 0.940 \\
\midrule
\multicolumn{5}{l}{\emph{Reference model Llama-3.2-3B}} \\
Embedding LightGBM & Plug-in & 0.023 & \textbf{0.023} & 0.000 \\
 & BoN-IPW & -0.011 & 0.194 & 0.920 \\
 & BoN-DR & 0.005 & 0.090 & 0.800 \\
TF--IDF logistic & Plug-in & 0.024 & \textbf{0.024} & 0.000 \\
 & BoN-IPW & -0.011 & 0.194 & 0.920 \\
 & BoN-DR & 0.006 & 0.084 & 0.820 \\
\midrule
\multicolumn{5}{l}{\emph{Reference model Qwen3.5-4B}} \\
Embedding LightGBM & Plug-in & 0.018 & \textbf{0.018} & 0.000 \\
 & BoN-IPW & 0.013 & 0.185 & 0.940 \\
 & BoN-DR & 0.003 & 0.061 & 0.830 \\
TF--IDF logistic & Plug-in & 0.014 & \textbf{0.014} & 0.000 \\
 & BoN-IPW & 0.013 & 0.185 & 0.940 \\
 & BoN-DR & 0.002 & 0.058 & 0.830 \\
\bottomrule
\end{tabular}

\end{table}

\section{Additional results for selection}
\label{app:selection-extra}

Table~\ref{tab:selection-challenging} gives the numerical results and median selected budgets for Figure~\ref{fig:selection}. The following subsections report the complete comparisons across configurations.

\begin{table}[H]
    \centering
    \caption{\textbf{Selecting $N$ under challenging conditions.} \textit{Synthetic experiments:} score error $\gamma=0.65$, reward function error $\delta=0.6$, $n=4{,}096$, $300$ repetitions; gain in policy value. \textit{GSM8K:} Qwen3.5-4B with OASST, $n=3{,}973$, $100$ logging repetitions; gain in accuracy points. Gain is measured against $N=1$ and higher is better; harm is the fraction of repetitions with negative gain and lower is better; $\widetilde{N}$ is the median selected budget. $N_{\max}=256$. \textbf{Bold} marks the best value per column among the selection methods. The last row gives the best attainable gain on the grid.}
    \label{tab:selection-challenging}
    \small
    \begin{tabular}{lrrrrrr}
\toprule
& \multicolumn{3}{c}{Synthetic exp (misspecified)} & \multicolumn{3}{c}{GSM8K: Qwen3.5-4B / OASST} \\
\cmidrule(lr){2-4}\cmidrule(lr){5-7}
Selection method & Gain & Harm & $\widetilde{N}$ & Gain (pp) & Harm & $\widetilde{N}$ \\
\midrule
Naive BoN & -0.388 & 1.00 & 256 & -1.50 & 1.00 & 256 \\
Plug-in & -0.380 & 1.00 & 256 & -0.28 & 1.00 & 4 \\
BoN-IPW, Value-max & 0.068 & \textbf{0.00} & 8 & -0.84 & 0.77 & 32 \\
BoN-IPW, Safe-improvement & 0.062 & \textbf{0.00} & 4 & \textbf{0.00} & \textbf{0.01} & 1 \\
BoN-DR, Value-max & \textbf{0.069} & \textbf{0.00} & 8 & -0.82 & 0.73 & 64 \\
BoN-DR, Safe-improvement & 0.067 & \textbf{0.00} & 4 & -0.19 & 0.14 & 1 \\
\midrule
\textit{Oracle-best budget} & \textit{0.071} & \textit{0.00} & \textit{8} & \textit{0.00} & \textit{0.00} & \textit{1} \\
\bottomrule
\end{tabular}

\end{table}

\subsection{Safe Improvement reduces harmful budget choices.}

Figure~\ref{fig:selection} expands the selection analysis in Figure~\ref{fig:selection_full} by reporting both the achieved gain and the harm rate in two challenging regimes. The Synthetic experiments with $\gamma=0.65$ and $\delta=0.6$ has an interior optimal budget, $N^\star=8$, so successful selection must exploit the benefit of increasing $N$ without overoptimizing at larger budgets. In contrast, for Qwen3.5-4B with OASST on GSM8K, $N^\star=1$, so any departure from the reference policy is harmful. These two regimes therefore separate the ability to capture a beneficial nontrivial budget from the ability to avoid unnecessary increases in inference-time compute.

\begin{figure}[H]
    \centering
    \includegraphics[width=\textwidth]{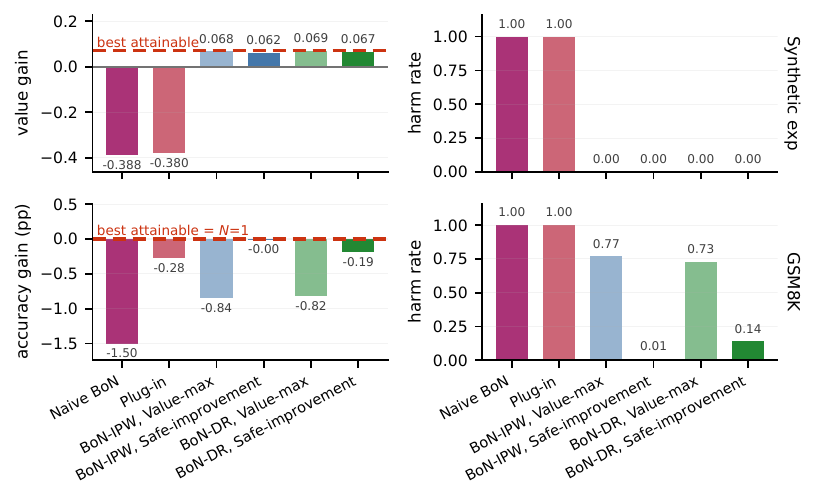}
    \caption{\textbf{Safe Improvement reduces harmful budget choices.} \textbf{Top:} Synthetic experiments with $\gamma=0.65$, $\delta=0.6$, $n=4{,}096$, and $300$ repetitions. \textbf{Bottom:} Qwen3.5-4B with OASST on GSM8K, with $3{,}973$ validation prompts and $100$ logging repetitions. \textbf{Left:} mean gain over $N=1$, measured in policy value for the Synthetic experiments and accuracy percentage points for GSM8K. Higher is better. The dashed red line marks the best attainable gain on the grid. \textbf{Right:} harm rate. Lower is better. Light and solid bars show Value-max and Safe-improvement, respectively. Plug-in uses Value Maximization, and Naive BoN selects $N_{\max}=256$. Labels report numerical values.}
    \label{fig:selection_full}
\end{figure}

In the Synthetic experiments, Value Maximization selection with BoN-IPW and BoN-DR recovers nearly all of the attainable gain and incurs no observed harm, whereas Naive BoN and Plug-in are harmful in every repetition. Safe Improvement is slightly more conservative, but retains most of the gain while likewise avoiding harmful selections. Thus, when the policy-value estimates are sufficiently accurate, both direct value maximization and uncertainty-aware selection identify useful nontrivial budgets.

The GSM8K regime is more challenging because the reference policy itself is optimal, so the best attainable gain over $N=1$ is zero. Value Maximization selection nevertheless frequently chooses larger, harmful budgets: BoN-IPW and BoN-DR lose $0.84$ and $0.82$ accuracy percentage points on average, with harm rates of $77\%$ and $73\%$, respectively. Safe Improvement substantially reduces these losses and lowers the harm rates to $1\%$ for BoN-IPW and $14\%$ for BoN-DR. The remaining BoN-DR harm is consistent with its finite-sample undercoverage at the largest budgets discussed in Appendix~\ref{app:real-coverage}. Thus, accounting for uncertainty is particularly useful when maximization over noisy value estimates would otherwise favor unnecessarily large budgets.

\subsection{Complete results on GSM8K}
\label{app:rule-matrix-gsm8k}

Table~\ref{tab:rule-matrix-gsm8k} reports every estimator with Value Maximization and Safe Improvement for all six configurations, including the aligned ones omitted from \rqthree. Two patterns hold throughout. Under the aligned Skywork reward model, no rule incurs harm, and Safe Improvement reduces gain without reducing the observed harm rate, most visibly for IPW. Under OASST, Safe Improvement reduces the harm incurred by Value Maximization selection with DR and IPW. Residual harm remains for DR on Llama and Qwen.

\begin{table}[H]
    \centering
    \caption{\textbf{Complete selection results on GSM8K.} $N_{\max}=256$, $100$ logging repetitions. Gain is measured in accuracy points against $N=1$ and higher is better; harm is the fraction of repetitions with negative gain and lower is better. Naive BoN does not depend on the selection rule and is reported once per configuration.}
    \label{tab:rule-matrix-gsm8k}
    \scriptsize
\begin{tabular}{rrl rr rr rr}
\toprule
\multicolumn{2}{c}{Naive BoN} & & \multicolumn{2}{c}{Plug-in} & \multicolumn{2}{c}{BoN-IPW} & \multicolumn{2}{c}{BoN-DR} \\
\cmidrule(lr){1-2}\cmidrule(lr){4-5}\cmidrule(lr){6-7}\cmidrule(lr){8-9}
Gain & Harm & Rule & Gain & Harm & Gain & Harm & Gain & Harm \\
\midrule
\multicolumn{9}{l}{\emph{OASST-RM-2.1, reference model Gemma-2-2B}} \\
\multirow{2}{*}{4.23} & \multirow{2}{*}{0.00} & Value-max & 4.84 & 0.00 & 4.45 & 0.00 & 4.51 & 0.00 \\
 &  & Safe-improvement & 4.84 & 0.00 & 3.25 & 0.00 & 3.98 & 0.00 \\
\midrule
\multicolumn{9}{l}{\emph{OASST-RM-2.1, reference model Llama-3.2-3B}} \\
\multirow{2}{*}{-0.51} & \multirow{2}{*}{1.00} & Value-max & 0.33 & 0.00 & -0.03 & 0.44 & -0.09 & 0.53 \\
 &  & Safe-improvement & 0.33 & 0.00 & 0.01 & 0.00 & -0.05 & 0.12 \\
\midrule
\multicolumn{9}{l}{\emph{OASST-RM-2.1, reference model Qwen3.5-4B}} \\
\multirow{2}{*}{-1.50} & \multirow{2}{*}{1.00} & Value-max & -0.28 & 1.00 & -0.84 & 0.77 & -0.82 & 0.73 \\
 &  & Safe-improvement & -0.28 & 1.00 & 0.00 & 0.01 & -0.19 & 0.14 \\
\midrule
\multicolumn{9}{l}{\emph{Skywork-Reward-V2, reference model Gemma-2-2B}} \\
\multirow{2}{*}{23.71} & \multirow{2}{*}{0.00} & Value-max & 23.71 & 0.00 & 21.86 & 0.00 & 22.91 & 0.00 \\
 &  & Safe-improvement & 23.71 & 0.00 & 16.17 & 0.00 & 20.57 & 0.00 \\
\midrule
\multicolumn{9}{l}{\emph{Skywork-Reward-V2, reference model Llama-3.2-3B}} \\
\multirow{2}{*}{15.92} & \multirow{2}{*}{0.00} & Value-max & 15.92 & 0.00 & 13.95 & 0.00 & 15.07 & 0.00 \\
 &  & Safe-improvement & 15.92 & 0.00 & 8.86 & 0.00 & 14.12 & 0.00 \\
\midrule
\multicolumn{9}{l}{\emph{Skywork-Reward-V2, reference model Qwen3.5-4B}} \\
\multirow{2}{*}{3.11} & \multirow{2}{*}{0.00} & Value-max & 3.11 & 0.00 & 2.57 & 0.00 & 2.91 & 0.00 \\
 &  & Safe-improvement & 3.11 & 0.00 & 0.38 & 0.00 & 2.23 & 0.00 \\
\bottomrule
\end{tabular}

\end{table}

\subsection{Complete results in the Synthetic experiments}
\label{app:rule-matrix-controlled}

Table~\ref{tab:rule-matrix-controlled} sweeps the reward function error at a fixed overoptimized score function. It shows why \rqone reports $\delta=0.6$: all estimator-based selection rules have zero observed harm up to $\delta=0.4$, because the Plug-in bias is not yet large enough to move its selected budget past the point at which the policy value falls below the value at $N=1$. At $\delta=0.6$ the Plug-in budget jumps to $N_{\max}$ and harms in every repetition, while DR and IPW retain positive gains with zero observed harm across the tested values of $\delta$. This also shows that a biased estimator can still select a safe budget, which is why estimation quality and selection quality are evaluated separately in \rqone.

\begin{table}[H]
    \centering
    \caption{\textbf{Complete selection results in the Synthetic experiments.} Score error $\gamma=0.65$, $n=4{,}096$, $N_{\max}=256$, $300$ repetitions. Gain is measured in policy value against $N=1$ and higher is better; harm is the fraction of repetitions with negative gain and lower is better. Naive BoN does not depend on the selection rule and is reported once per value of $\delta$.}
    \label{tab:rule-matrix-controlled}
    \scriptsize
\begin{tabular}{rrl rr rr rr}
\toprule
\multicolumn{2}{c}{Naive BoN} & & \multicolumn{2}{c}{Plug-in} & \multicolumn{2}{c}{BoN-IPW} & \multicolumn{2}{c}{BoN-DR} \\
\cmidrule(lr){1-2}\cmidrule(lr){4-5}\cmidrule(lr){6-7}\cmidrule(lr){8-9}
Gain & Harm & Rule & Gain & Harm & Gain & Harm & Gain & Harm \\
\midrule
\multicolumn{9}{l}{\emph{Reward estimator error $\delta=-0.2$}} \\
\multirow{2}{*}{-0.388} & \multirow{2}{*}{1.00} & Value-max & 0.065 & 0.00 & 0.068 & 0.00 & 0.070 & 0.00 \\
 &  & Safe-improvement & 0.065 & 0.00 & 0.062 & 0.00 & 0.067 & 0.00 \\
\midrule
\multicolumn{9}{l}{\emph{Reward estimator error $\delta=0$}} \\
\multirow{2}{*}{-0.388} & \multirow{2}{*}{1.00} & Value-max & 0.071 & 0.00 & 0.068 & 0.00 & 0.070 & 0.00 \\
 &  & Safe-improvement & 0.071 & 0.00 & 0.062 & 0.00 & 0.067 & 0.00 \\
\midrule
\multicolumn{9}{l}{\emph{Reward estimator error $\delta=0.2$}} \\
\multirow{2}{*}{-0.388} & \multirow{2}{*}{1.00} & Value-max & 0.071 & 0.00 & 0.068 & 0.00 & 0.070 & 0.00 \\
 &  & Safe-improvement & 0.071 & 0.00 & 0.062 & 0.00 & 0.067 & 0.00 \\
\midrule
\multicolumn{9}{l}{\emph{Reward estimator error $\delta=0.4$}} \\
\multirow{2}{*}{-0.388} & \multirow{2}{*}{1.00} & Value-max & 0.049 & 0.00 & 0.068 & 0.00 & 0.070 & 0.00 \\
 &  & Safe-improvement & 0.049 & 0.00 & 0.062 & 0.00 & 0.067 & 0.00 \\
\midrule
\multicolumn{9}{l}{\emph{Reward estimator error $\delta=0.6$}} \\
\multirow{2}{*}{-0.388} & \multirow{2}{*}{1.00} & Value-max & -0.380 & 1.00 & 0.068 & 0.00 & 0.069 & 0.00 \\
 &  & Safe-improvement & -0.319 & 1.00 & 0.062 & 0.00 & 0.067 & 0.00 \\
\bottomrule
\end{tabular}

\end{table}

\subsection{Selection across the range of score error}
\label{app:score-sweep}

Figure~\ref{fig:syn-score-selection} varies the score error at a fixed reward function. Naive BoN is optimal while the score function is aligned and collapses as soon as it is exploitable, whereas every estimator-based rule has nonnegative mean gain across the tested range. DR attains the lowest regret against the optimal budget.

\begin{figure}[H]
    \centering
    \includegraphics[width=\textwidth]{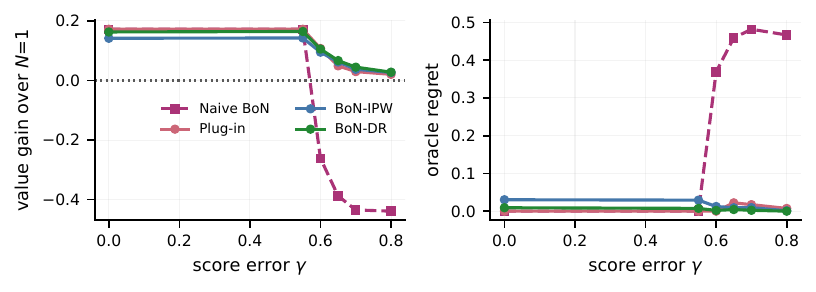}
    \caption{\textbf{Selection across the range of score error in the Synthetic experiments.} Safe Improvement rule, $\delta=0.4$, $n=4{,}096$, $N_{\max}=256$, $300$ repetitions. \textbf{Left:} gain over $N=1$. \textbf{Right:} regret against the optimal budget on the grid, lower is better.}
    \label{fig:syn-score-selection}
\end{figure}

\section{Extension to preference rewards}
\label{sec:alignment_extension}

Our BoN evaluation framework extends to settings in which the scalar reward encodes a pairwise preference between two model responses, as in AI alignment and RLHF. For a prompt or context $X$, let
\[
    A = (A_1,A_2)
\]
denote a pair of candidate responses and let
\[
    Y \in \{0,1\}
\]
be the observed reward indicating whether the first response is preferred to the second response. Thus,
$Y=1$ means that $A_1$ is preferred to $A_2$.

We assume that the behavior comparison policy factorizes as
\[
    \pi_b(a_1,a_2\mid x)
    =
    \pi_{b,1}(a_1\mid x)\pi_{b,2}(a_2\mid x).
\]
Here, $\pi_{b,1}$ denotes the policy from which the first candidate response is
sampled, and $\pi_{b,2}$ denotes the policy from which the second, reference
candidate response is sampled. In an RLHF application, $\pi_{b,1}$ could be the
base, supervised fine-tuned, or previous model from which BoN samples are
drawn. The policy $\pi_{b,2}$ is the opponent or reference policy used in the
comparison. Often, the two policies are the same, so that
\[
    \pi_{b,1} = \pi_{b,2} = \pi_{\mathrm{ref}},
\]
corresponding to the common design in which both candidate responses are sampled
independently from the same reference model. More generally, however, the two
marginals may differ, for example if the first response is sampled from a current
model and the second response from a fixed baseline model.

Let $s(x,a)$ denote the score function used for BoN selection. Given a prompt $x$,
the BoN policy induced by $\pi_{b,1}$ samples
\[
    \widetilde A_1,\ldots,\widetilde A_N
    \overset{\mathrm{iid}}{\sim}
    \pi_{b,1}(\cdot\mid x)
\]
and returns
\[
    \begin{gathered}
    J_N(x):=\argmax_{1\leq j\leq N}s(x,\widetilde A_j),\\
    I_N(x)\mid\widetilde A_1,\ldots,\widetilde A_N
    \sim\operatorname{Uniform}(J_N(x)),
    \qquad A_N^\star:=\widetilde A_{I_N(x)}.
    \end{gathered}
\]
Ties are thus broken uniformly among maximizing sampled occurrences, including repeated responses. We denote the resulting policy by $q_N(\cdot\mid x)$.

\subsection{Pairwise Win-Rate Objective}

In the preference-learning setting, a natural target estimand is the pairwise
win rate of the BoN policy against a reference policy $\rho$:
\[
    W_N(\rho)
    =
    \mathbb{E}_X
    \mathbb{E}_{A\sim q_N(\cdot\mid X),\,B\sim \rho(\cdot\mid X)}
    \left[
        \mu(X,A,B)
    \right],
\]
where
\[
    \mu(x,a,b)
    =
    \mathbb{P}(Y=1\mid X=x,A_1=a,A_2=b)
\]
is the conditional mean reward function, here the probability that response $a$ is preferred to response $b$.

The reference policy $\rho$ specifies the opponent against which the BoN
policy is evaluated. A particularly convenient and practically relevant choice is
\[
    \rho = \pi_{b,2}.
\]
In this case, the target asks how often a BoN response sampled from the
first behavior marginal beats a response sampled from the same reference
distribution that generated the second response in the logged preference data.
When
\[
    \pi_{b,1}=\pi_{b,2}=\rho=\pi_{\mathrm{ref}},
\]
the estimand is the probability that a BoN response beats an ordinary
response sampled from the same reference model.

A common modeling assumption in preference learning is the Bradley--Terry model,
under which there exists a latent reward function $r_0(x,a)$ such that
\[
    \mu(x,a,b)
    =
    \sigma\{r_0(x,a)-r_0(x,b)\},
\]
where $\sigma(u)=(1+\exp(-u))^{-1}$ is the logistic function. Given a sample
$\{(X_i,A_{i1},A_{i2},Y_i)\}_{i=1}^n$, a reward model $r_\theta$ can be estimated
by minimizing the negative log-likelihood
\[
\begin{aligned}
    \mathcal{L}_{\mathrm{BT}}(\theta)
    ={}&
    -
    \frac{1}{n}
    \sum_{i=1}^n
    \Bigl[
        Y_i
        \log
        \sigma\{r_\theta(X_i,A_{i1})-r_\theta(X_i,A_{i2})\}
        \\
        &\qquad
        +
        (1-Y_i)
        \log
        \sigma\{r_\theta(X_i,A_{i2})-r_\theta(X_i,A_{i1})\}
    \Bigr].
\end{aligned}
\]
The corresponding reward function is
\[
    \widehat\mu(x,a,b)
    =
    \sigma\{\widehat r(x,a)-\widehat r(x,b)\}.
\]
As in Section~\ref{sec:theory}, fitted functions are constructed independently of the evaluation data and auxiliary draws. The Bradley--Terry model is useful because it connects pairwise preferences to
the reward modeling perspective common in RLHF. However, the pairwise win-rate
estimand does not require the Bradley--Terry model for identification. It is
identified as long as the conditional preference regression
$\mu(x,a,b)$ is identified on the relevant support.

\subsection{Doubly Robust Estimation}

The target pair distribution for the win-rate estimand is
\[
    q_N^{\mathrm{pair}}(a,b\mid x)
    =
    q_N(a\mid x)\rho(b\mid x),
\]
whereas the behavior pair distribution is
\[
    \pi_b(a,b\mid x)
    =
    \pi_{b,1}(a\mid x)\pi_{b,2}(b\mid x).
\]
Therefore, the pairwise density ratio is
\[
    w_N(x,a,b)
    =
    \frac{
        q_N(a\mid x)\rho(b\mid x)
    }{
        \pi_{b,1}(a\mid x)\pi_{b,2}(b\mid x)
    }.
\]

Define
\[
    m_N(x;\rho)
    =
    \mathbb{E}_{A\sim q_N(\cdot\mid x),\,B\sim \rho(\cdot\mid x)}
    \left[
        \mu(x,A,B)
    \right].
\]
Then the efficient influence function for $W_N(\rho)$ in the nonparametric model
for $\mu$ is
\[
    \phi_N(O)
    =
    m_N(X;\rho)
    -
    W_N(\rho)
    +
    w_N(X,A_1,A_2)
    \{Y-\mu(X,A_1,A_2)\}.
\]
This yields the doubly robust estimator
\[
    \widehat W_N(\rho)
    =
    \frac{1}{n}
    \sum_{i=1}^n
    \left[
        \widehat m_N(X_i;\rho)
        +
        \widehat w_N(X_i,A_{i1},A_{i2})
        \{Y_i-\widehat\mu(X_i,A_{i1},A_{i2})\}
    \right].
\]
The plug-in component can be estimated by Monte Carlo:
\[
    \widehat m_N(x;\rho)
    =
    \frac{1}{M}
    \sum_{m=1}^M
    \widehat\mu(x,\widetilde A_{m,N}^\star,\widetilde B_m),
\]
where
\[
    \widetilde A_{m,N}^\star=\widetilde A_{m,I_{m,N}},
    \qquad
    \widetilde A_{mj}
    \overset{\mathrm{iid}}{\sim}
    \pi_{b,1}(\cdot\mid x),
    \qquad
    \widetilde B_m
    \sim
    \rho(\cdot\mid x).
\]

Here, conditional on the sampled candidates, $I_{m,N}$ is drawn uniformly from $\argmax_{1\leq j\leq N}s(x,\widetilde A_{mj})$. The $M$ terms above are independent Monte Carlo replicates, each using $N$ candidate draws and an independent opponent draw; this is a direct Monte Carlo construction rather than the pooled estimator of Section~\ref{sec:bon_dr_estimator}.

\subsection{BoN density ratios}

The density ratio can be simplified using the structure of the BoN
policy. Suppose $q_N$ is induced by sampling $N$ candidates independently from
$\pi_{b,1}(\cdot\mid x)$ and selecting the response with the largest score
$s(x,a)$. Define
\[
\begin{aligned}
    F_{x,1}^-(a)&=\prob_{A\sim\pi_{b,1}(\cdot\mid x)}\{s(x,A)<s(x,a)\},\\
    F_{x,1}(a)&=\prob_{A\sim\pi_{b,1}(\cdot\mid x)}\{s(x,A)\leq s(x,a)\}.
\end{aligned}
\]
Proposition~\ref{prop:bon-dr-score} gives, for $\pi_{b,1}(a\mid x)>0$,
\[
    \frac{q_N(a\mid x)}{\pi_{b,1}(a\mid x)}
    =
    g_N\!\left(F_{x,1}(a),F_{x,1}^-(a)\right).
\]
Thus,
\[
    w_N(x,a,b)
    =
    g_N\!\left(F_{x,1}(a),F_{x,1}^-(a)\right)
    \cdot
    \frac{\rho(b\mid x)}{\pi_{b,2}(b\mid x)}.
\]

The strict and weak score probabilities can be estimated by Monte Carlo sampling from
the first behavior marginal. Given samples
\[
    \bar A_1,\ldots,\bar A_L
    \overset{\mathrm{iid}}{\sim}
    \pi_{b,1}(\cdot\mid x),
\]
we estimate
\[
\begin{aligned}
    \widehat F_{x,1}^-(a)&=\frac1L\sum_{\ell=1}^L\1\{s(x,\bar A_\ell)<s(x,a)\},\\
    \widehat F_{x,1}(a)&=\frac1L\sum_{\ell=1}^L\1\{s(x,\bar A_\ell)\leq s(x,a)\}.
\end{aligned}
\]
This gives the plug-in estimate of the BoN ratio
\[
    \widehat R_N(x,a)=g_N\!\left(\widehat F_{x,1}(a),\widehat F_{x,1}^-(a)\right).
\]
This plug-in ratio estimate is generally biased at finite $L$. The unbiased alternative is Eq.~\eqref{eq:weight-pooled}, applied to the first behavior marginal with $M=L\geq N-1$; the plug-in ratio estimate does not inherit its finite-pool unbiasedness. The resulting pairwise weight estimate is
\[
    \widehat w_N(x,a,b)
    =
    \widehat R_N(x,a)
    \cdot
    \frac{\rho(b\mid x)}{\pi_{b,2}(b\mid x)}.
\]

A particularly important special case is
\[
    \rho=\pi_{b,2}.
\]
Then the second density ratio cancels and
\[
    w_N(x,a,b)
    =
    \frac{q_N(a\mid x)}{\pi_{b,1}(a\mid x)}
    =
    g_N\!\left(F_{x,1}(a),F_{x,1}^-(a)\right).
\]
Hence the doubly robust estimator becomes
\[
    \widehat W_N(\pi_{b,2})
    =
    \frac{1}{n}
    \sum_{i=1}^n
    \left[
        \widehat m_N(X_i;\pi_{b,2})
        +
        \widehat R_N(X_i,A_{i1})
        \{Y_i-\widehat\mu(X_i,A_{i1},A_{i2})\}
    \right].
\]

If, additionally,
\[
    \pi_{b,1}=\pi_{b,2}=\rho=\pi_{\mathrm{ref}},
\]
then both logged responses are independent samples from the same reference model,
the BoN policy is obtained by reranking samples from that same model, and
the target value is
\[
    W_N(\pi_{\mathrm{ref}})
    =
    \mathbb{E}_{X,A\sim q_N(\cdot\mid X),\,B\sim \pi_{\mathrm{ref}}(\cdot\mid X)}
    \left[
        \mu(X,A,B)
    \right].
\]
In words, $W_N(\pi_{\mathrm{ref}})$ is the probability that a BoN
response beats a standard response from the reference model. In this common
case, the only density ratio required for debiasing is the BoN ratio
\[
    R_N(x,a)
    =
    g_N\!\left(F_{x,1}(a),F_{x,1}^-(a)\right),
\]
which can be estimated using samples from the reference model alone.

This extension shows that the proposed BoN evaluation approach carries
over naturally to preference-learning settings. The key requirement is that the
logged comparison policy factorizes, and that the first behavior marginal is
the policy from which BoN candidates are sampled. If the behavior pair
distribution does not factorize, then the required density ratio generally
depends on the full joint distribution of $(A_1,A_2)$ and cannot be recovered
from samples from the marginal policies alone.

\end{document}